%% file: ET_arxiv1.tex
\documentclass[11pt,reqno]{amsart}
\usepackage[cp1251]{inputenc}
\usepackage[english]{babel}
\usepackage{amsmath}
\usepackage{amssymb}
\usepackage{amsfonts}
\usepackage{graphicx}
\usepackage{xcolor}
\usepackage[colorlinks]{hyperref}

\input{preamble}

\usepackage{titlesec}
\titlespacing{\section}{0pt}{2.5ex}{1.5ex}
\titlespacing{\subsection}{0pt}{1.5ex}{1ex}
\titlespacing{\subsubsection}{0pt}{1.5ex}{1ex}
\titleformat{\section}{\large\bfseries\centering}{\thesection}{1em}{}
\titleformat{\subsection}[runin]{\bfseries}{\thesubsection.}{0.5em}{}[.\mbox{\ }]
\titleformat{\subsubsection}[runin]{\bfseries}{\thesubsubsection.}{0.4em}{}[.\mbox{\ }]

\newcommand\orcidicon[1]{\href{https://orcid.org/#1}{\includegraphics[scale=0.02]{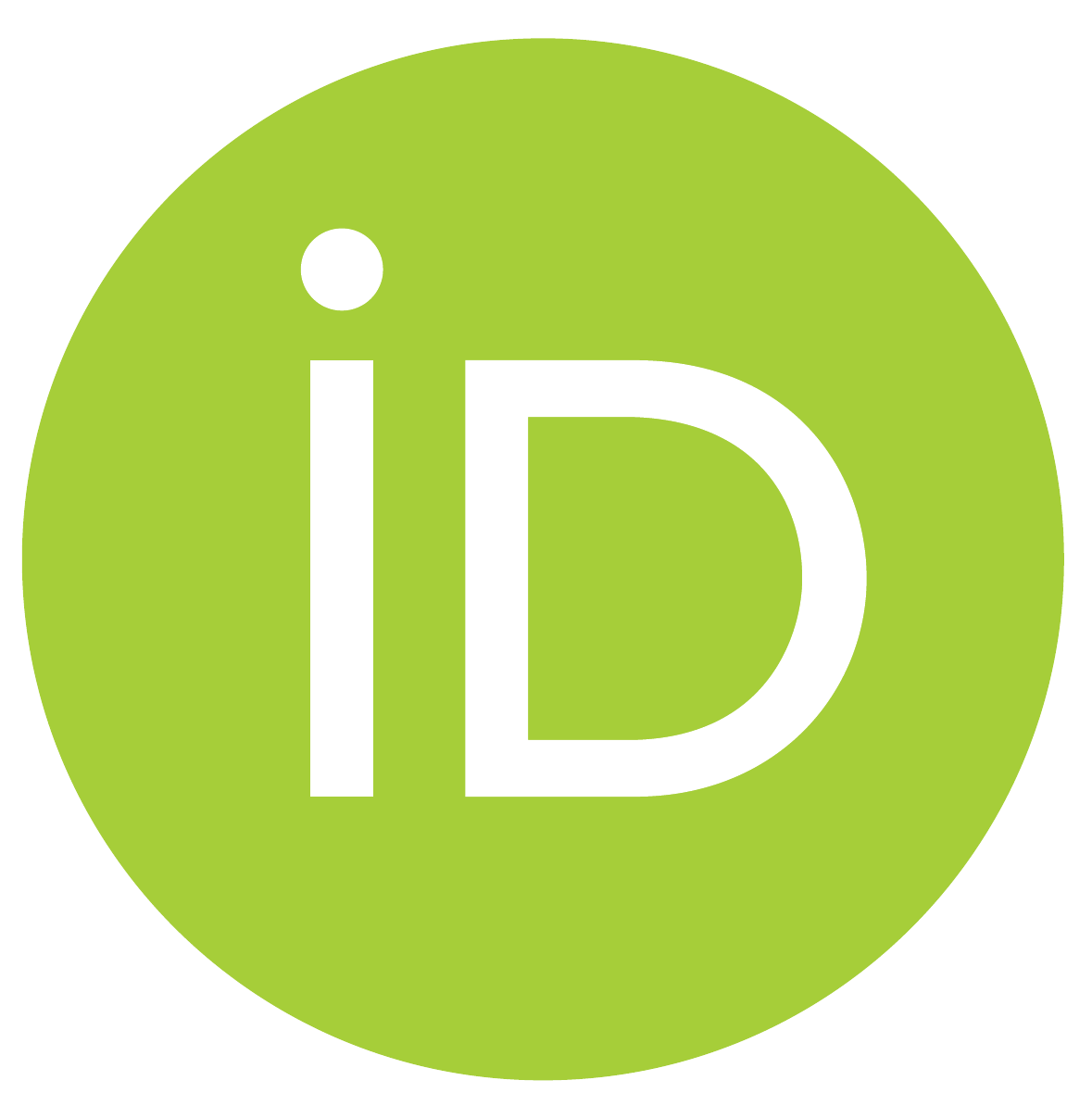}}}

\newtheorem{lemma}{Lemma}
\newtheorem{theorem}{Theorem}
\newtheorem{definition}{Definition}
\newtheorem{corollary}{Corollary}
\newtheorem{proposition}{Proposition}
\definecolor{myblue}{rgb}{0.03, 0.27, 0.79}

\hypersetup{linkcolor=blue,filecolor=black,urlcolor=blue, citecolor=blue}

\begin{document}
\renewcommand{\refname}{References}

\thispagestyle{empty}

\title[\sc Escape Time from Local Optima]{\large 
Asymptotical Analysis of the $(1+(\lambda,\lambda))$~GA Escape Time from Local Optima on Jump Functions}

\author[A.V. EREMEEV, V.A. TOPCHII]{{\bf A.V. Eremeev}\protect\orcidicon{0000-0001-5289-7874}{\bf V.A. Topchii}\protect\orcidicon{0000-0003-4310-5665}}

\address{Antion Valentinovich Eremeev  % Контактные данные всех авторов и место работы указываются только на английском языке
\newline\hphantom{iii} Sobolev Institute of Mathematics,
\newline\hphantom{iii} pr. Koptyuga, 4,
\newline\hphantom{iii} 630090, Novosibirsk, Russia;%
\newline\hphantom{iii} Dostoevsky Omsk State University,
\newline\hphantom{iii} 6, Frunze str., 
\newline\hphantom{iii} 644099, Omsk, Russia}%
\email{\textcolor{blue}{eremeev@ofim.oscsbras.ru}}%
\address{Valentin Alekseevich Topchii  % Контактные данные всех авторов и место работы указываются только на английском языке
\newline\hphantom{iii} Sobolev Institute of Mathematics,
\newline\hphantom{iii} pr. Koptyuga, 4,
\newline\hphantom{iii} 630090, Novosibirsk, Russia}%
\email{\textcolor{blue}{topchij@ofim.oscsbras.ru}}%

\thanks{\sc Eremeev A.V., Topchii V.A.
Escape Time from Local Optima}
%\thanks{\copyright \ 2023 Eremeev A.V., Topchii V.A.}
\thanks{\rm The research was supported by Russian Science Foundation
grant N 25-21-00335, https://rscf.ru/project/25-21-00335/.}
%\thanks{\it Received January, 1, 2023, Published December, 31, 2023}%

\vspace{1cm}
\maketitle

\bigskip
\bigskip

\begin{quote}
\noindent{\bf Abstract:} The paper develops the approach to the runtime analysis of evolutionary algorithms on the basis of limit theorems from probability theory. We consider the family of Jump$_k$ benchmark functions, defined on the search space of binary strings of length~$n$, parametrized by the integer~$k$, which have a plateau of multiple local optima at the Hamming distance~$k$ from a unique global optimum. 

In this work, we consider the genetic algorithm $(1+(\lambda,\lambda))~GA$ from (Doerr, Doerr and Ebel, 2015) with tunable parameters of the mutation rate~$p$, crossover bias~$c$, and two intermediate population sizes $\lambda_M$ and $\lambda_C$. 
We study the time it escapes from the plateau of local optima and reaches the global optimum in the case of Jump$_k$ fitness function
and tighten the upper bounds on the expected escape time, known from the work of Antipov, Doerr and Karavaev (2022). 
The obtained bounds also apply to a wider range of algo\-rith\-mic parameters.
The main result of this work applies to the case when $k\to \infty$ as $n \to \infty.$ The case of finite $k$ is investigated quite simply and considered tangentially. \\ 

 \end{quote}

\bigskip

\section{Introduction}
This work contains an extended version of the theoretical analysis 
presented in the report at the GECCO-26 conference (see \cite{GEC26}).

Given a pseudo-Boolean maximization problem
$f:\{0,1\}^n \to \mathbb{R}$, an evolutionary algorithm~(EA) seeks for the maximum of~$f$, using a population of one or more individuals, corresponding to the points in the solution space $\{0,1\}^n,$ and the {\em fitness} of the individuals is equal to the value of the objective
function $f(x)$. New test points in the EA are produced using
the random operators of  {\em mutation} and {\em crossover}.
When using the latter, the EA is usually called
{\em a genetic algorithm}. One of the most
frequently used mutation operators is the {\it standard
mutation}~\cite{Gold89}, where each bit of the given string $\mathrm{x}\in
\{0,1\}^n$,  independently of the other bits, changes its value with a given probability $p$. In this paper, we will assume that in the case of standard mutation, at each iteration of the EA, the number of  bits~$\ell$ to be mutated is selected, using the binomial distribution $\Bin(n, p)$ (next, we denote $\ell \in {\rm Bin}(n,p)$). Each new descendant is obtained from
the parent individual by making changes in randomly selected
$\ell$ bits.

The authors of~\cite{DDE15} developed the $(1 + (\lambda, \lambda))$~genetic algorithm (GA) for pseudo-Boolean optimization, where a crossover operator is able to  eliminate some ``unsuccessful'' mutations. At each iteration of the $(1+(\lambda, \lambda))$~GA, $\lambda$ offspring are generated from a single parent individual independently of each other, at the Hamming distance $\ell$ from the parent, $\ell \in {\rm Bin}(n,p)$.
%$n$ is the length of the bit-string and $p$ is the mutation parameter.
After that, the best individual in terms of fitness is selected from these $\lambda$ solutions and the crossover operator is applied, combining it with the parent, given a crossover parameter~$c$:  With probability~$c$ the crossover operator uses bits from the best child, and with probability $1-c$, it uses bits from the parent solution. This way $\lambda$ individuals are created, and the best of these $\lambda$ individuals is accepted as a new parent, provided it is at least as fit as the previous parent. Otherwise the current individual remains unchanged.
Theoretical analysis of optimization time showed that the algorithm $(1 + (\lambda, \lambda))$~GA for many values of configurable parameters turns out to be asymptotically faster on the \onemax fitness function than most classical evolutionary algorithms.
Here and below, $\onemax(x):=\sum_{i=1}^n x_i$ for any $\mathrm{x}\in\{0,1\}^n$.

Formally, we consider here a more general 
$(1+(\lambda_{M}, \lambda_{C}))$~GA, where $\lambda_{M}$ and $\lambda_{C}$ are some, 
possibly different, sequences of population sizes at mutation and crossover stage, respectively, but the main applications are traditionally given for the case of $\lambda_{M}=\lambda_{C}=\lambda$. Dropping the $\lambda_{M}=\lambda_{C}=\lambda$ condition naturally leads to a very wide class of asymptotic results. To simplify the notation, we will refer to $(1 + (\lambda_{M}, \lambda_{C}))$~GA as $(1 + (\lambda, \lambda))$~GA.

%In particular, as shown in~\cite{DD18} for the \onemax, the optimal choice of fixed mutation parameter~$p$ 
%for the entire running time gives
%the optimization time
%$$
%{E[T]=\Theta\left(n \sqrt{{\log(n)\log\log\log(n)}/{\log\log n}}\right).}
%$$ 

To describe the asymptotic behaviour of quantities associated with
unbounded growth of the dimension~$n$, we will use
the standard notations: $g(n)=O(t(n))$ and $g(n)=o(t(n))$, if the upper limit of the ratio $t(n)/g(n)$ is bounded or converges to 0, respectively; ${g(n)\sim t(n)}$, if $g(n)=t(n)(1+o(1))$. When using the latest asymptotic relations, we will usually omit the formal condition ``for $n\to\infty$''.
Avoiding the convention accepted in computer science (see, e.g.,~\cite{Cormen}, Ch.~3), we {\em do not} assume that the 
value of $g(n)=o(t(n))$ is non-negative or non-negative for sufficiently large~$n$.
%It is sayed that a function $t(n)$ belongs to the set $ \Theta(g(n))$ if there exist positive constants $c_1$ and $c_2$, as well as a non-negative integer $n_0$ such that $c_2 g(n) \le t(n) \le c_1 g(n)$ for all $n\ge n_0$ (see, e.g.,~\cite{Cormen}, Ch.~3).

In~\cite{ADK22},  the first mathematical runtime analysis for the $(1+(\lambda, \lambda))$~GA optimizing a multimodal optimization problem, namely the the classical Jump$_k$ benchmark, was carried out. It was observed there that a combination of aggressive mutation with crossover as repair mechanism works even better on this benchmark. 
The $(1+(\lambda, \lambda))$~GA can optimize jump functions with gap size $k \le n/4$ in expected time at most $n^{(k+1)/2}e^{O(k)}k^{-k/2},$ 
while the runtime of many classic mutation-based algorithms is of order~$n^k$.
For the problem of leaving the local optimum of \jump, the authors of~\cite{ADK22} perform parameters tuning that leads to a
runtime $(n/k)^{k/2}e^{O(k)}$.

It is shown in~\cite{ABD24} that choosing {\em all} parameter values  (population size $\lambda$, mutation rate $p$, and crossover bias $c$) in each iteration randomly, using power-law distribution, is a robust approach. It was proved that this algorithm
on the one hand can imitate simple hill-climbers on problems like \onemax, \leadingones, or Minimum Spanning
Tree. On the other hand, the algorithm is also efficient on \jump functions, where the best static parameters are very different from those necessary to optimize simple problems.

In the present paper, we consider the $(1+(\lambda, \lambda))$~GA from~\cite{DDE15} with tunable parameters of mutation rate~$p$, crossover bias~$c$, and two intermediate population sizes $\lambda_M$ and $\lambda_C$, and study the time to escape from the local optima in the case of Jump$_k$ fitness function when $np$  tends to infinity. 
The main result of this work is a tightened upper bound on the time to escape a local optimum $(n/k)^{k/2}e^{O(k)}$ from~\cite{ADK22}.
Besides that, the obtained bound applies to a wider range of algorithm's parameters.

\section{Preliminaries}

For the sake of compactness, the number of ones in a bit string $\mathrm{x}$ will be denoted as $|\mathrm{x}|$, meaning the $L^{1}$ norm of the vector~$\mathrm{x}$, or Hamming norm.

When introducing new notations in formulas, the symbol $:=$ is used, and the colon is on the side of the quantity being defined. 
Let us denote by $\mathbb{N}_{0,n}$ the set of all integers from 0 to $n$.

The notation $\ell\in\mathcal{B}(n,p)$, where $n\in \mathbb{N}$ and $p\in(0,1)$, means that the random variable $\ell$ has a binomial distribution with probability distribution
$$
\mathbf{P}(\ell=l)=\genfrac(){0pt}{0}{n}{l}p^{l}q^{n-l}=:P_{p}(n,l), \ l\in\mathbb{N}_{0,n}:=\{0,1,\ldots,n\},
$$
where $q=1-p$. 
%Throughout what follows, we assume that the probability $p=p(n)$ is bounded away from~1. In other words, there exists $\epsilon_{p}\in(0,1)$ such that $p<1-\epsilon_{p}$. It is well known that $\mathbf{E}[\ell]=np$ and ${\rm Var}[\ell]=npq$.

The probability to have $s$ successes in $l$ random draws, without replacement, from a finite set of $n$ objects, that contains exactly $k$ objects with a required feature is described by the hypergeometric distribution. The probability mass function of a random variable that follows the hypergeometric distribution is given by
$$
P_{n,k}(l,s):={{k \choose s}{n-k \choose l-s}}\Big/{{n \choose l}}.
$$

\subsection{The $(1+(\lambda,\lambda))$~Genetic Algorithm}

We consider the same version of the $(1+(\lambda,\lambda))$~GA as in~\cite{ADK22}, which is analogous to the one initially proposed in~\cite{DDE13}. The algorithm has the following outline. 

{\em \textbf{Algorithm 1.} $(1+(\lambda,\lambda))$~GA  with the offspring population sizes $\lambda_M$ and $\lambda_C$, mutation rate~$p$,  and crossover bias $c$,
maximizing an objective function~$f:\{0,1\}^{n}\to \mathbb{R}$.}

\medskip 
\ \,1. $\mathrm{x}$ $\leftarrow$ random bit string of length $n$;

\ \,2. \ \textbf{while} \textit{not terminated} \textbf{do}

\ \,3.  \quad Mutation phase:

\ \,4.  \quad Choose $\ell \in {\rm Bin}(n,p)$;

\ \,5.  \quad \textbf{for} $i\in[1..\lambda_M]$ \textbf{do}

\ \,6.  \quad   \quad  $\mathrm{x}^{(i)}$ $\leftarrow$ a copy of $\mathrm{x}$;

\ \,7.  \quad   \quad   Flip $\ell$ bits in $\mathrm{x}^{(i)}$ chosen uniformly at random;

\ \,8.  \quad \textbf{end}

9.  \ \quad $\mathrm{x}' \leftarrow {\rm arg\, max}_{\mathrm{z}\in
\{\mathrm{x}^{(1)},...,\mathrm{x}^{(\lambda_M)}\}} f(\mathrm{z})$;

10.   \quad Crossover phase:

11.  \quad \textbf{for} $i\in[1..\lambda_C]$ \textbf{do}

12.  \quad   \quad  Create $\mathrm{y}^{(i)}$ by taking each bit from $\mathrm{x}'$ with   

13. \ \ \, \;  \quad   \quad probability $c$ and from $\mathrm{x}$ with probability $1-c$;

14.  \quad \textbf{end}

15.  \quad  $\mathrm{y} \leftarrow {\rm arg\, max}_{\mathrm{z}\in
\{\mathrm{y}^{(1)},...,\mathrm{y}^{(\lambda_C)}\}} f(\mathrm{z})$;

16.  \quad  \textbf{if} $f(\mathrm{y}) \geq f(\mathrm{x})$ \textbf{then}

17.   \quad   \quad $\mathrm{x} \leftarrow \mathrm{y}$;

18.   \quad \textbf{end}

19. \textbf{end}
\medskip 

In what follows, we will refer to execution of Lines 3--18 of Algorithm~1 as {\em one iteration} of the algorithm. Note that in the one iteration may be considered in two perspectives: (i)~treating $\ell$ as a random variable, or (ii) considering a specific sample $\ell=l$. Usage of notation $\ell$ or $l$ will clearly identify either of these two options below.  
We will refer to Lines~3--9 of Algorithm~1 as the {\em mutation phase,} and Lines~10--15 as the {\em crossover phase}.

\subsection{Jump Functions}

The basic benchmark function
$
\onemax(\mathrm{x})=\sum\nolimits_{i=1}^{n} x_i,
$
has been deeply studied in the literature on the theory of EAs.  In terms of \onemax it is easy to define the function $\jump_k,$ where $k$ is a parameter $2\le k\le n$:
\begin{align*} 
  &\jump_k(\mathrm{x})
    := \begin{cases}
         n + 1     & \text{if } \onemax(\mathrm{x})=n\\
         k + \onemax(x) & \text{if } \onemax(\mathrm{x})\leq n-k\\
         n - \onemax(\mathrm{x}) & \text{otherwise}
        \end{cases}.
\end{align*}

The global optimum of $\jump_k$ corresponds to the bit string of all ones, which we denote by $\mathrm{x_{m}}$. The infinite family of functions $\jump_k$ serves as an example of a multimodal function, where all local optima are equally distant from the unique global optimum and a number of papers are devoted to the runtime analysis of evolutionary algorithms on these functions, see the references in~\cite{ADK22}.

\section{Technical Definitions }

Let an individual $\mathrm{x}$ have the norm $|\mathrm{x}|=n-k$, i.e. $\mathrm{x}$ is in a local maximum of $\jump_k$.

In what follows, $p_{M}^{(1)}(n,l,k,p)$ denotes
probability that $k$ zero bits and $l-k$ one bits are flipped during a single mutatioin Line 7 of Algorithm 1 to the individual $\mathrm{x}$. We denote this event by $M^{(1)}(n,l,k,p)$. When the event $M^{(1)}(n,l,k,p)$ occurs, the resulting individual $\widetilde{\mathrm{x}}$ has the norm $|\widetilde{\mathrm{x}} |=n-l+k$. Otherwise, $\widetilde{\mathrm{x}}$ has the norm $|\widetilde{\mathrm{x}}|<n-l+k$.

The explicit form of the probability of event $M^{(1)}(n,l,k,p)$ is
\begin{equation} \label{main4}
p_{M}^{(1)}(n,l,k,p)=P_{n,k}(l,k)=\dfrac{(n-k)!l!}{(l-k)!n!},\ \ k\leq l,
\end{equation}
here $P_{n,k}(l,s)$ is the hypergeometric distribution and we assume $P_{n,k}(l,k)=0$ for $l<k$. 
\newline
Let $p_{M}^{(\lambda_{M})}(n,l,k,p)$ be the probability that, given $\lambda_{M}$ independent mutations of $\mathrm{x}$, the event $M^{(1)}(n,l,k,p)$ occurs at least once throughout the mutation phase. We denote this event by $M^{(\lambda_{M})}(n,l,k,p)$. The event $M^{(\lambda_{M})}(n,l,k,p)$ leads to the choice of the individual $\mathrm{x}'$ with the norm $|\mathrm{x}'|=n-l+k$. If not all $k$ positions with ``0'' components in $\mathrm{x}$ have mutated, then $\widetilde{\mathrm{x}}$ has a norm $|\widetilde{\mathrm{x}}|<n-l+k$.
According to this notation, the following relation holds:
\begin{equation} \label{Mlam}
p_{M}^{(\lambda_{M})}(n,l,k,p)=1-\Big(1-p_{M}^{(1)}(n,l,k,p)\Big)^{\lambda_{M}}.
\end{equation}

Let $p_{C}^{(1)}(n,l,k,p,c)=c^{k}(1-c)^{l-k}$ be the probability that in the crossover phase, within a single crossover of $\mathrm{x}$ and $\mathrm{x}'$ with $|\mathrm{x}'|=n-l+k$ in Lines~12, 13, only the $k$ zero-components in $\mathrm{x}$ are replaced (i.e. these components are selected from $\mathrm{x}'$). We denote this event by $C^{(1)}(n,l,k,p,c)$.

Let $p_{C}^{(\lambda_{C})}(n,l,k,p,c)$ be the probability that throughout the crossover phase, conditioned that the event $M^{(\lambda_{M})}(n,l,k,p)$ occurred (i.e., in the case of $|\mathrm{x}|=n-k$ and $|\mathrm{x}'|=n-l+k$), the event $C^{(1)}(n,l,k,p,c)$ will occur at least once. We denote this event by $C^{(\lambda_{C})}(n,l,k,p,c)$.
According to this notation, we have
\begin{eqnarray} \label{Clam}
p_{C}^{(\lambda_{C})}(n,l,k,p,c){=}1{-}\Big(1-p_{C}^{(1)}(n,l,k,p,c)\Big)^{\lambda_{C}}{=}
1-\big(1{-}c^{k}(1-c)^{l-k}\big)^{\lambda_{C}}.
\end{eqnarray}

Let us denote by $A_{\lambda_{M},\lambda_{C}}(n,l,k,p)= M^{(\lambda_{M})}(n,l,k,p)C^{(\lambda_{C})}(n,l,k,p,c)$ the event in which an iteration of the algorithm results in  creating the optimum, i.e. we get $\mathrm{y}=\mathrm{x_{m}}$.

By definition, $p_{\lambda_{M},\lambda_{C}}(n,l,k,p):=\mathbf{P}(A_{\lambda_{M},\lambda_{C}}(n,l,k,p))$ is the probability of reaching the optimum as a result of an iteration, it is equal to
\begin{equation*} %\label{1lam}
p_{\lambda_{M},\lambda_{C}}(n,l,k,p,c)=p_{M}^{(\lambda_{M})}(n,l,k,p)p_{C}^{(\lambda_{M})}(n,l,k,p,c).
\end{equation*}
 
The transition from a non-random $l$ to a random~$\ell$ is quite natural:
we average it out using the total probability formula for a binomial distribution with probabilities $P_{p}(n,l).$
Let us introduce the following  notation for the resulting expectations: 
\begin{eqnarray*} %\label{1lam0l}
&p_{M}^{(\lambda_{M})}(n,k,p):=\mathbf{E}\big[ p_{M}^{(\lambda_{M})}(n,\ell,k,p)\big],\\ %\label{1lam02}
&p_{C}^{(\lambda_{C})}(n,k,p,c):=\mathbf{E}\big[ p_{C}^{(\lambda_{C})}(n,\ell,k,p,c)\big], \\ %\label{1lam03}
&p_{\lambda_{M},\lambda_{C}}(n,k,p,c):=\mathbf{E}[p_{\lambda_{M},\lambda_{C}}(n,\ell,k,p,c)].
%\hspace{0.2cm}
\end{eqnarray*}

The following proposition appears to be interesting by itself, as it provides simple expressions for $p_{M}^{(\lambda_{M})}(n,k,p)$
and $p_{C}^{(\lambda_{C})}(n,k,p,c)$, 
for the special case of $\lambda_{M}=1$ and $\lambda_{C}=1$.

\begin{proposition} \label{le000}
Suppose $\ell\in \mathcal{B}(n,p)$, then 
\begin{eqnarray} \label{main000}
&p_{M}^{(1)}(n,k,p)=\mathbf{E}[P_{n,k}(\ell,k)]=p^{k},
\\\label{main002}
&p_{C}^{(1)}(n,k,p,c)=\mathbf{E}[c^{k}(1-c)^{\ell-k}]=c^{k} (1-c)^{-k} (1-pc)^{n}.
\end{eqnarray}
\end{proposition}
\begin{proof}
The generating function of the binomial distribution is 
$$
\phi_{\ell}(z):=\mathbf{E}[z^{\ell}]=(q+pz)^{n}, \quad|z|\leq1.
$$ 

The factorial moment of order $k$ is 
$$
V^{(k)}: =\mathbf{E}\Big[\prod\nolimits_{i=0}^{k-1}(\ell-i)\Big]=\dfrac{n!}{(n-k)!}\mathbf{E}[P_{n,k}(\ell,k)].
$$

By definition, $V^{(k)}$ equals to the $k$-th derivative of $\phi_{\ell}(z)$ at $z=1$, which is $p^{k}n!/(n-k)!$. This completes the proof of identity~\eqref{main000}.

Since the generating function of the binomial distribution is $\mathbf{E}[x^{l}]=(q+px)^{n}$,
so $\mathbf{E}[(1-c)^{l}]=(q+p(1-c))^{n}$, which proves identity \eqref{main002}.
\end{proof}

\section{Asymptotic Approach}

Our goal  is to obtain upper bounds on the expected number of fitness evaluations and the expected number of iterations till the global optimum is reached, if the algorithm is started in a local optimum, i.e. the initial solution is such that $|x|=n-k$. Let us denote these random variables by $T_{F}$ and $T_{I}$, respectively.

The overall proof scheme is quite natural and similar to the proof of Lemmas~6 and~7 and Theorem~8 from \cite{ADK22}. We present our approach in this section.

Suppose that on some set $B=B(n)\subseteq \mathbb{N}_{n}:=\{ 1,2,\ldots,n\}$ of some positive probability $\mathbf{P}(B)>C_{B}>0$ the following bound is valid
\begin{equation} \label{M0l}
p_{\lambda_{M}}(n,l,k,p)>C_{M}>0,\ l\in B.
\end{equation}

Here the sequences $C_{B}=C_{B}(n)>0$ and $C_{M}=C_{M}(n)>0$ may depend on~$n$. 
If the condition \eqref{M0l} is satisfied and the following estimate holds
$$
\sum\nolimits_{l\in B}p_{C}^{(\lambda_{C})}(n,l,k,p,c)P_{p}(n,l)\geq C_{C,B}>0,
$$
for binomial distribution $P_{p}(n,l)$ and some sequence $C_{C,B}{=}C_{C,B}(n)>0$ then the inequality~\eqref{1lam0l1} below describes the probability of reaching the optimum as a result of one iteration:
\begin{equation} \label{1lam0l1}
p_{\lambda_{M},\lambda_{C}}(n,k,p,c) \geq  C_{M} \sum_{l\in B}p_{C}^{(\lambda_{C})}(n,l,k,p,c)P_{p}(n,l)\geq C_{M}C_{C,B}>0.
\end{equation}
The expected number of iterations  $\mathbf{E}[T_{I}]$ and the expected number of fitness evaluations $\mathbf{E}[T_{F}]$ till reaching the optimum may be expressed as
\begin{equation} \label{TI0}
\mathbf{E}[T_{I}]=p_{\lambda_{M},\lambda_{C}}^{-1}(n,k,p,c)\leq \frac{1}{C_{M}C_{C,B}},\ \mathbf{E}[T_{F}]\leq \frac{\lambda_{M}+\lambda_{C}}{C_{M}C_{C,B}}.
\end{equation}
Here we have used the fact that $T_{I}$ has a geometric distribution with a probability of success $p_{\lambda_{M},\lambda_{C}}(n,k,p,c)$.

If $ p_{\lambda_{C}}(n,l,k,p,c)>C_{C}>0$, $l\in B$, then analogously to~\eqref{1lam0l1}, one can obtain the lower bound
\begin{equation} \label{1lam0l01}
p_{\lambda_{M},\lambda_{C}}(n,k,p,c) \geq C_{C} \sum\nolimits_{l\in B}p_{M}^{(\lambda_{M})}(n,l,k,p)P_{p}(n,l) \geq   C_{C}C_{M,B}>0
\end{equation}
for some sequence $C_{M,B}=C_{M,B}(n)>0.$
In particular, if $p_{C}^{(\lambda_{C})}(n,l,k,p,c)\geq C_{C}$ and $p_{M}^{(\lambda_{M})}(n,l,k,p)\geq C_{M}$ for $l\in B$, then
\begin{equation} \label{TI11}
p_{\lambda_{M},\lambda_{C}}(n,k,p,c)\geq C_{C}C_{M}C_{B}, \ \ \mathbf{E}[T_{I}]\leq C_{M}^{-1}C_{C}^{-1}C_{B}^{-1}.
\end{equation}

Let us discuss the choice of the set $B$. In \cite{ADK22} (in the proof of Lemma~6) it is chosen in the form 
%$B\subseteq\mathbb{N}$, 
$B=[np,2np]$, where $np\geq2k$. We propose to choose it in the form $B\subseteq[\max\{k,np - N\sqrt{npq}\},np + N\sqrt{npq}]$, where $N\in(0,\infty)$. In this paper, we will analyze only the case $k\leq np - N \sqrt{npq}$ and in the final asymptotic steps we set $B=\Delta_{n}(N):=[np - N\sqrt{npq},np + N\sqrt{npq}]$. Dropping the $k\leq np - N \sqrt{npq}$ condition is easy, but it will lead to cumbersome calculations.

By the de Moivre–Laplace Theorem (see e.g. Theorem~5.3.1 from~\cite{Bo99})
for $N=N(n)=o\big((npq)^{1/6}\big)$ (possibly dependent on $n$) we have
\begin{equation} \label{1}
\mathbf{P}(\ell\in\Delta_{n}(N))\sim 2\Phi(N)-1=: c_{\Phi}(N),
\end{equation}
where $\Phi(N)=(2\pi)^{-0.5}\int_{-\infty}^{N}e^{-u^{2}/2}\mathsf{d}u$ is the cumulative distribution function of the standard normal random value. In the case of $N\to\infty$ we get
\begin{equation*}
\mathbf{P}(\ell\in\Delta_{n}(N))\to1, \ \ \mathbf{P}(\ell\notin\Delta_{n}(N))\to0.
\end{equation*}
Note that $2\Phi(2)-1= 0.99730020$ for $N=2$ and
$2\Phi(5)-1=0.99999942$ for $N=5$, and then converges very quickly to 1. But increasing of $N$ worsens the estimates for $p_{C}^{(\lambda_{C})}(n,l,k,p,c)\geq C_{C}$ and $p_{M}^{(\lambda_{M})}(n,l,k,p)\geq C_{M}$ for $l\in B$. 

The  length of the interval $\Delta_{n}(N)$ is Infinitesimal to $[np,2np]$, this is the reason why the upper bounds on the runtime obtained by means of $B=[np,2np]$ are by more than a factor of $2^{k(1+o(1))}$ higher than ours (for more details, see the comments below the proof of Lemma \ref{le01}). 
%In what follows, we present explicit asymptotic form of the constants, rather than comparing them by the order of magnitude.

\section{Asymptotic Bounds}

\subsection{Asymptotics of Mutation and Crossover}

Let us introduce the notations 
$$
\epsilon_{\mu}(x):= \sum_{i\geq \mu+1}\frac{x^{i}}{i(i-1)},\quad \mu\in \mathbb{N},\quad \text{where}\ \ \sum_{i=2}^{1}\frac{x^{i}}{i(i-1)}:=0,
$$ 
and consider the asymptotic behavior of $p_{M}^{(1)}(n,l,k,p)$ as $n\to\infty$. First, we present some technical results.

\begin{lemma} \label{le1}
Let $x\in[0,1)$, $m,\mu\in\mathbb{N}$ and $\epsilon_{\mu}(x):=\sum_{i\geq \mu+1}x^{i}(i(i-1))^{-1}$, then the following representations are valid 
\begin{eqnarray} \label{main1}
&&(1-x)^{m(1-x)}=\exp\Big\{m\Big(-x+\sum_{i\geq 2}x^{i}(i(i-1))^{-1}\Big)\Big\} \\ \label{main2}
&=:&\exp\bigg\{m\bigg(-x+\sum_{i=2}^{\mu}x^{i}(i(i-1))^{-1}+\epsilon_{\mu}(x)\bigg)\bigg\},\\ \label{main02}
&&\hspace{-2.2cm}\mbox{where } \ 0\leq\epsilon_{\mu}(x)= \dfrac{\theta x^{\mu+1}}{(1-x)\mu(\mu+1)},\ \theta\in(0,1).
\end{eqnarray}
\end{lemma}
\begin{proof}
The Taylor series expansion of the logarithm for $|x|<1$ is
\begin{equation} \label{ln1}
\ln(1-x)=-\sum_{i\geq 1}x^{i}/i.
\end{equation}
Noting that the identity $(1-x)^{y}=e^{y\ln(1-x)}$, in view of the representation
\begin{equation*}
-(1-x)\ln(1-x)=(1-x)\sum_{i\geq 1}x^{i}/i=x-\sum_{i\geq 2}x^{i}(i(i-1))^{-1},
\end{equation*}
which is valid for absolutely convergent series if $x\in[0,1)$, after substitution $y=m(1-x)$, entails the statement~\eqref{main1}.
In the expression~\eqref{main2}, $\epsilon_{\mu}(x)= \sum_{i> \mu}x^{i}(i(i-1))^{-1}$, therefore after using the formula for the sum of a geometric progression $\sum_{i>\mu}x^{i}=x^{\mu+1}/(1-x)$ we obtain the bound~\eqref{main02}.
\end{proof}

Let us describe some relevant properties of the functions $\epsilon_{1}(x)$
\begin{lemma} \label{le03}
The functions $\epsilon_{1}(x)$ and $\epsilon_{1}(x)/x$, $x\in[0,1)$, are monotonically increasing, $\epsilon_{1}(x),\epsilon_{1}(x)/x\in[0,1)$ and the following relations hold:
\begin{equation} \label{main5}
\epsilon_{1}(x)=(1-x)\ln(1-x)+x=\sum_{i=2}^{\infty}\dfrac{x^{i}}{i(i-1)},
\end{equation}
and $\epsilon_{1}'(x){=}-\ln(1-x)>0$.

If $x,y\in[0,1)$ and $x>y$, then the following inequalities hold:
\begin{equation} \label{main5D}
0.5(x-y)\leq D(x,y)\leq (1/6+(3(1-x))^{-1})(x-y)\leq 0.5\dfrac{x-y}{1-x}, 
\end{equation}
where $D(x,y):=\epsilon_{1}(x)/x-\epsilon_{1}(y)/y$.
\end{lemma}
\begin{proof}
The relations \eqref{main5} are written out by definition, where the series converges absolutely on the interval $x\in[0,1)$, and the function $(1-x)\ln(1-x)$ at zero is defined by continuity, i.e., equal to zero. Therefore $\epsilon_{1}(0)=0$ and $\epsilon_{1}(1)=1$.

By definition of $D(x,y)$ and by the formula for the sum of the geometric progression, we obtain 
\begin{equation*}
D(x,y)=\sum_{i=2}^{\infty}\dfrac{x^{i-1}-y^{i-1}}{i(i-1)}=(x-y)\sum_{i=2}^{\infty}\dfrac{\sum\nolimits_{j=0}^{i-2}x^{j}y^{i-2-j}}{i(i-1)}.
\end{equation*}

Replacing all variables in the last representation with $y$ or $x$, we obtain the inequalities 
\begin{eqnarray*}
&0.5(x-y)\leq(x-y){\displaystyle\sum\limits_{i=2}^{\infty}}\dfrac{y^{i-2}}{i}\leq D(x,y)\leq (x-y){\displaystyle\sum\limits_{i=2}^{\infty}}\dfrac{x^{i-2}}{i}\\
&\leq \left (0.5+\dfrac{x}{3(1-x)}\right )(x-y)\leq 0.5\dfrac{x-y}{1-x},
\end{eqnarray*}
equivalent to \eqref{main5D}, which is the last statement of the lemma.
\end{proof}

Recall the Stirling formula (see e.g. \cite{Bo99} p.~117, Section~2.2).
\begin{equation*} %\label{Stir}
s!=\sqrt{2\pi s}s^{s}e^{-s}R(s), \ \ s\in\mathbb{N},
\end{equation*} 
for the function $R(s)$ that satisfies the inequalities $1<e^{\frac{1}{12s+1}}<R(s)<e^{\frac{1}{12s}}<1.08690405$. Note that $R(s)=1+o(1)$ if $s\to \infty$.

\begin{lemma} \label{le3}
Suppose $k\leq (1-\varepsilon)n$ and $l\geq(1+\varepsilon)k$ for some $\varepsilon\in(0,1)$, then the following representation is valid: 
\begin{equation} \label{main03}
p_{M}^{(1)}(n,l,k,p)=\sqrt{\dfrac{1-k/n}{1-k/l}}e^{-kD(k/l,k/n)}(l/n)^{k} R_{n,l,k},
\end{equation}
where $R_{n,l,k}=R(n-k)R(l)R^{-1}(l-k)R^{-1}(n)$. 
\end{lemma}
\begin{proof}

Using the Stirling formula, and taking into account the expression~\eqref{main4}, we obtain
\begin{eqnarray} \nonumber
&&p_{M}^{(1)}(n,l,k,p)=\dfrac{(n-k)!l!}{(l-k)!n!}=\sqrt{\dfrac{(n- k)l}{(l-k)n}}\dfrac{(n-k)^{n-k}l^{l}}{(l-k)^{l-k}n^{n}}R_{n,l,k}\\ \label{bpM}
&=& \sqrt{\dfrac{1-k/n}{1-k/l}}\dfrac{(1-k/n)^{n-k}}{(1-k/l)^{l-k}}(l/n)^{k} R_{n,l,k}. %\qquad \qquad \qquad \quad 
\end{eqnarray}
By Lemma~\ref{le1} 
\begin{equation} \label{bpM1}
\dfrac{(1-k/n)^{n-k}}{(1-k/l)^{l-k}}=\dfrac{\exp\{- n(k/n{-}\epsilon_{1}(k/n))\}}{\exp\{-l(k/l{-} \epsilon_{1}(k/l))\}} 
= \exp\{n\epsilon_{1}(k/n){-}l\epsilon_{1}(k/l)\}.
\end{equation}
Assertion~\eqref{main03} of this lemma follows from   Lemma \ref{le03} and equations \eqref{bpM} and \eqref{bpM1}.
\end{proof}

We will analyze the properties of the second bound from formulas~\eqref{TI0}. 

For the sake of brevity, let us temporarily denote $\lambda_{M}$ and $\lambda_{C}$ by $L_{1}$ and $L_{2}$, $p_{M}^{(1)}(n,l,k,p)$ and $p_{C}^{(1)}(n,l,k,p,c)$ by $p_{1}=p_{1}(n)$ and $p_{2}=p_{2}(n)$, and $p_{M}^{(\lambda_{M})}(n,l,k,p)$ and $p_{C}^{(\lambda_{C})}(n,l,k,p,c)$ by $p_{1}^{(L_{1})}$ and $p_{2}^{(L_{2})}$.
Then the relations \eqref{Mlam} and \eqref{Clam} imply 
\begin{equation} \label{e1}
1-e^{-L_{i}p_{i}}\leq p_{i}^{(L_{i})}=1-(1-p_{i})^{L_{i}}=1-e^{L_{i}\ln(1-p_{i})},\ \ i=1,2,
\end{equation}
and
\begin{equation} \label{0e1}
p_{i}^{(L_{i})}=1-e^{-L_{i}p_{i}(1+o(1))},\ \ i=1,2,
\end{equation}
in the case $0<p_{i}=p_{i}(n)\to 0$, for $n\to\infty$. Note that by \eqref{0e1}
\begin{equation} \label{111} 
p_{i}^{(L_{i})}=1-e^{L_{i}\ln(1-p_{i})}=L_{i}p_{i}(1+o(1))
\end{equation}
as $L_{i}p_{i}\to0$.

Let us denote the main part of $\mathbf{E}[T_{F}]$ in the case $p_{i}=o(1)$, $i=1,2$,
\begin{equation*} %\label{e2}
\varphi_{L_{1},L_{2}}(n):= \dfrac{L_{1}+L_{2}}{(1-e^{-L_{1}p_{1}})(1-e^{-L_{2}p_{2}})}.
\end{equation*}

\begin{lemma} \label{le10}
Let conditions \eqref{0e1} be satisfied, where the sequences $p_{i}=p_{i}(n)>0$, $i=1,2$, are fixed, 
and the representations $L_{i}=(\widetilde{c}_{i} +o(1))p_{i}^{-1}$, $\widetilde{c}_{i}\in\mathbb{R}^{+}$, are valid, then the function 
$$
\widetilde{\varphi}_{L_{1},L_{2}}(n)=\dfrac{L_{1}+L_{2}}{\big(1{-}e^{-\widetilde{c}_{1}}\big)\big(1{-}e^{-\widetilde{c}_{2}}\big)}=\varphi_{L_{1},L_{2}}(n)(1+o(1))
$$ 
is increasing of $\widetilde c_{i}$.

If $L_{1}(n)=L_{2}(n)=L(n)$,
$L(n)p_{2}(n)\to 0$, and $\widetilde{c}_{1}>0$, then
\begin{equation} \label{mam1}
\varphi_{L,L}(n)=\dfrac{2L(n)}{L(n)p_{2}(n)(1-e^{-\widetilde{c}_{1}})}(1+o(1)). 
\end{equation}
\end{lemma}
\begin{proof}
We give an asymptotically equivalent form for the function $\varphi_{L_{1},L_{2}}(n)$ for $\widetilde{c}_{i}>0$
\begin{equation*} %\label{mam12}
\dfrac{\widetilde{c}_{1}/p_{1}^{(L_{1})}+\widetilde{c}_{2}/p_{2}^{(L_{2})}}{(1-e^{-\widetilde{c}_{1}}) (1-e^{-\widetilde{c}_{2}})}.
\end{equation*}

The function $x/(1-e^{-x})$ is monotonically increasing, so according to estimates \eqref{TI0}, the average number of fitness evaluations decreases as $\widetilde{c}_{i}>0$ tends to zero.
Relation \eqref{mam1} clearly follows from the conditions of the lemma.
\end{proof}

Let us describe the exact asymptotic behavior of $p_{M}^{(\lambda_{M})}(n,l,k,p)$ for $n\to\infty$, $\lambda_{M}\to\infty$, and a number of additional constraints. The imposed conditions can be relaxed, but this will lead to cumbersome calculations and formulations of results.

Let us introduce the notation $\mathcal{P}_{k,p}(c_{1})$ for the set of sequences $\{k\}=\{k(n)\}_{n\in\mathbb{N}}$ and $\{p\}=\{p(n)\in(0,p_{0})\subseteq(0,1)\}_{n\in\mathbb{N}}$ that satisfy the conditions
\begin{equation*} %\label{restr}
\varlimsup_{n\to\infty}k(n)p^{-1}(n)/n\leq c_{1},
\end{equation*}
where $c_{1}\in[0,1)$ and $p_{0}\in[0,1)$  are fixed numbers.

Note that under the condition $\mathcal{P}_{k,p}(c_{1})$ the estimation 
$$
\min_{l\in\Delta_{n}(N)}l-k\geq (1-c_{1})np(1+o(1)), \ \ n\to\infty, 
$$
holds.
Moreover, the following inequalities are valid
\begin{equation} \label{e01}
1\leq\sqrt{\frac{1-k/n}{1-k/(np)}}\leq(1-c_{1})^{-1}(1+o(1))=O(1).
\end{equation}

In order to formulate several useful lower bounds on the probability 
of ``successful'' mutations in Lemma~\ref{le01}  below, let us first
introduce
\begin{definition}\label{def:J_MnkpN}
Let
\begin{eqnarray} \label{cole01}
 J_{M,n,k,p,N}:=\frac{\sqrt{\frac{1-k/n}{1-k/(np)}} 
\cdot \ p^{k}}{\exp\left \{\dfrac{kN\sqrt{q}}{\sqrt{np}}+\dfrac{kN^{2}q}{2np}+\dfrac{k^{2}(q+N\sqrt{pq}/\sqrt{n})}{2np\big(1-N\sqrt{q}/\sqrt{np}-k/(np)\big)}\right \}}.
\end{eqnarray}
\end{definition}

Note that in the case of $k=o(np)$, according to this definition, $J_{M,n, k,p,N} \sim p^{k}e^{-Ck^{2}/(np)}\gg (p/2)^{k}$, where $C$ is a certain constant.

\begin{lemma} \label{le01} 
Suppose that $k\in\mathcal{P}_{k,p}(c_{1})$,  $p\in(0,p_{0})$ with any fixed $c_{1}\in[0,1)$, $N\in\mathbb{N}$, $p_{0}\in(0,1)$.  Let $k\to\infty$, $np\to\infty$, $l\in\Delta_{n}(N)\subseteq [k+1,n]$
the following inequality is satisfied
\begin{equation} \label{010}
\lambda_{M}p^{k}\exp\left \{-\dfrac{k^{2}q}{2np}(1+o(1))\right \}=o(1),
\end{equation}
then for $n\to\infty$ the following asymptotic relations hold:
\begin{eqnarray} \label{1main3}
p_{M}^{(\lambda_{M})}(n,l,k,p)\geq \lambda_{M} J_{M,n,k,p,N}(1+o(1)),\\ \label{0main3}
p_{M}^{(\lambda_{M})}(n,k,p)\geq \lambda_{M}c_{\Phi}(N)J_{M,n,k,p,N} (1+o(1)).
\end{eqnarray}

If, instead of constraint \eqref{010}, the conditions
\begin{equation} \label{010e}
\lambda_{M}J_{M,n,k,p,N}{\geq} c_{M}^{(0)}(1{+}o(1)), \  l\in\Delta_{n}(N),
\end{equation}
or
\begin{equation} \label{010e00}
\lambda_{M}J_{M,n,k,p,N}\to\infty,\  l\in\Delta_{n}(N),
\end{equation}
are satisfied, where $ c_{M}^{(0)}>0$ is an absolute constant, then for $n\to\infty$ the following asymptotic relations hold 
\begin{eqnarray} \label{1main30}
&p_{M}^{(\lambda_{M})}(n,l,k,p)\geq  (1-e^{-c_{M}^{(0)}})(1+o(1)),\ l\in\Delta_{n}(N) ,\\ \label{0main30}
&p_{M}^{(\lambda_{M})}(n,k,p)\geq c_{\Phi}(N)(1-e^{-c_{M}^{(0)}})(1+o(1)),
\end{eqnarray}
under the condition \eqref{010e} or
\begin{equation} \label{1main3000}
1-p_{M}^{(\lambda_{M})}(n,l,k,p)=o(1),\quad
p_{M}^{(\lambda_{M})}(n,k,p)\geq c_{\Phi}(N)(1+o(1))
\end{equation}
under the condition \eqref{010e00}. 
\end{lemma}
\begin{proof} 
The conditions of the lemma imply that $p$ is separated from 1. We give an explicit form for the probability of the event $M^{(\lambda_{M})}(n,l,k,p)$ using the representations \eqref{Mlam} and \eqref{main03} from Lemma \ref{le3}. Under the assumptions of the lemma, the conditions $p_{M}^{(1)}(n,l,k,p)\to0$ are satisfied for $n\to\infty$.
Therefore, for $n\to\infty$, for all $l$ satisfying the conditions of the lemma the following representation holds: 
\begin{eqnarray} \label{main3}
&&p_{M}^{(\lambda_{M})}(n,l,k,p)=1{-}\exp\{\lambda_{M}\ln(1-p_{M}^{(1)}(n,l,k,p))\} \\ \nonumber
&&=1-\exp\Bigg\{- \lambda_{M}
\sqrt{\dfrac{1-k/n}{1-k/l}}\dfrac{R(l)}{R(l-k)}\bigg(\frac{l}{n}\bigg)^{k}e^{-kD\big(\frac{k}{l},\frac{k}{n}\big)}(1+o(1))\Bigg\}.
\end{eqnarray}

In the lemma conditions we assume that $k\in\mathcal{P}_{k,p}(c_{1})$, $np\to\infty$ and $l\in\Delta_{n}(N)\subseteq [k+1,n]$. Then the conditions $l,l-k\to\infty$ are satisfied and the fraction $R(l)/R(l-k)$ in \eqref{main3} can be omitted since it will be included in the factor $1+o(1)$. By analogy, it is obvious that the estimate $\sqrt{1-k/l}=\sqrt{1-k/np}(1+o(1))$ holds.

The proof of the first part of the lemma is based on the estimates for $\overline{\pi}:= \sup_{l\in\Delta_{n}(N)}p_{M}^{(1)}(n,l,k,p)$ and $\underline{\pi}:=\inf_{l\in\Delta_{n}(N)}p_{M}^{(1)}(n,l,k,p)$. 
Let us prove that the relation $\lambda_{M}\overline{\pi}=o(1)$, if the condition \eqref{010} is satisfied, and with the help of this the inequality $\underline{\pi}\geq J_{M,n,k,p,N}(1+o(1))$.  

Let us move on to the estimates of $l^{k}/n^{k}$ and $D\big(k/l,k/n\big)$.

The $l\in\Delta_{n}(N)$ condition allows us to write the inequalities
\begin{eqnarray*}
&\dfrac{\big(np-N\sqrt{npq}\big)^{k}}{n^{k}}\leq l^{k}/n^{k}\leq\dfrac{\big(np+N\sqrt{npq}\big)^{k}}{n^{k}},\\
&D\bigg(\dfrac{k}{np+N\sqrt{npq}},\dfrac{k}{n}\bigg)\leq D\bigg(\dfrac{k}{l},\dfrac{k}{n}\bigg)\leq D\bigg(\dfrac{k}{np-N\sqrt{npq}},\dfrac{k}{n}\bigg).
\end{eqnarray*}

According to the Taylor formula for logarithm function, the following representations are valid in the neighborhood of the point 1
\[
\dfrac{\big(np\mp N\sqrt{npq}\big)^{k}}{n^{k}}{=}p^{k}e^{k\ln\left (1\mp \frac{N\sqrt{q}}{\sqrt{np}}\right )}
{=}p^{k}\exp\left \{\mp \frac{kN\sqrt{q}}{\sqrt{np}} -\dfrac{0.5kN^{2}q(1{+}o(1))}{np}\right \}.
\]

As a result, we get
\begin{equation}  \label{as1}
p^{k}\exp\left\{- \frac{kN\sqrt{q}}{\sqrt{np}}-\frac{0.5kN^{2}q(1{+}o(1))}{np}\right\}\leq l^{k}/n^{k}\leq p^{k}e^{\frac{kN\sqrt{q}}{\sqrt{np}}}.
\end{equation}

Using the identity 
\[
\dfrac{k}{np(1\pm x)}-\dfrac{k}{n}=\dfrac{kq}{np}\mp\dfrac{kN\sqrt{q}}{(np)^{3/2}(1\pm x)}
\]
for $x=N\sqrt{q}/\sqrt{np}\in(0,0.5)$ we have
\begin{eqnarray*}
\dfrac{k}{np+ N\sqrt{npq}}-\dfrac{k}{n}&\geq&\dfrac{kq}{np}-{\dfrac{kN\sqrt{q}}{(np)^{3/2}(1+x)}\geq\dfrac{kq}{np}-\dfrac{kN\sqrt{q}}{(np)^{3/2}},}\\
\dfrac{k}{np- N\sqrt{npq}}-\dfrac{k}{n}&\leq&\dfrac{kq}{np}+\dfrac{kN\sqrt{q}}{(np)^{3/2}(1-x)}\leq\dfrac{kq}{np}+\dfrac{2kN\sqrt{q}}{(np)^{3/2}},
\end{eqnarray*}
and the inequalities \eqref{main5D}, we obtain the estimates 
\begin{eqnarray} \nonumber
{\dfrac{kq}{2np}-\dfrac{kN\sqrt{q}}{2(np)^{3/2}}}&\leq& D\left(\frac{k}{l},\frac{k}{n}\right)\leq  \dfrac{2k(q+2N\sqrt{q}/\sqrt{np})}{2np\big(1-k\big(np- N\sqrt{npq}\big)^{-1}\big)}\\ \label{as3}
&\leq& c_{1}\dfrac{q+2N\sqrt{q}/\sqrt{np}}{2\big(1-c_{1}\big(1-N\sqrt{q/(np}\big)^{-1}\big)}=O(1),
\end{eqnarray}
where did we take into account that $np\to\infty$ and $1-N\sqrt{q/(pn}\big)^{-1}=1+o(1)$.

To apply relations \eqref{e01}, \eqref{main3}, \eqref{as1}, \eqref{as3} and \eqref{111} in the first part of Lemma, the estimations 
\begin{eqnarray} \nonumber
&0\leq\lambda_{M}\overline{\pi}\leq O(1)\sup_{l\in\Delta_{n}(N)} \lambda_{M}\big(l/n\big)^{k} e^{-kD\big(k/l,k/n\big)}\leq O(1)\lambda_{M}p^{k}\\ \label{as5}
&\cdot\ \exp\left \{\frac{kN\sqrt{q}}{\sqrt{np}}-\frac{k^{2}q(1{+}o(1))}{2np}\right\}=O(\lambda_{M})p^{k}\exp\left \{-\frac{k^{2}q(1{+}o(1))}{2np}\right \}=o(1)\ \ \
\end{eqnarray}
must be true. 
Here, on the two last staps we used the fact that the estimates $\frac{kN\sqrt{q}}{\sqrt{np}}-\frac{k^{2}q(1+o(1))}{2np}=O(1)$ are valid for $k/\sqrt{np}=O(1)$, $k/\sqrt{np}=o(k^{2}/(np))$ are valid for $k/\sqrt{np}\to\infty$ and then the condition \eqref{010}.

If the proof of relation \eqref{as5} is based on upper estimates for $p_{M}^{(1)}(n,l,k,p)$, $l\in\Delta_{n}(N)$, then to prove statement \eqref{1main3} we need these lower estimates. The schemes of their proofs coincide, but when estimating from the above double inequalities, in them it is necessary to take the opposite. 

Using inequalities \eqref{as1} and \eqref{as3}, we obtain the following estimates  
\begin{eqnarray} \nonumber
&&\Big(\frac{l}{n}\Big)^{k}\exp\Big\{{-}kD\Big(\frac{k}{l},\frac{k}{n}\Big)\Big\}\geq p^{k}\exp\left \{{-}\dfrac{kN\sqrt{q}}{\sqrt{np}}-\dfrac{kN^{2}q(1+o(1))}{2np}\right \}
\\ \nonumber
&\cdot& \exp\bigg \{-0.5\dfrac{k^{2}(q+N\sqrt{pq}/\sqrt{n})}{np\big(1-\frac{k}{np}\big(1- N\sqrt{q/(np)}\big)^{-1}\big)\big(1-N\sqrt{q}/\sqrt{np}\big)}\bigg\}   \\  \label{00as3}
&=& \frac{p^{k}(1+o(1))}{\exp\left \{
\dfrac{kN\sqrt{q}}{\sqrt{np}}+\dfrac{kN^{2}q}{2np}+\dfrac{k^{2}(q+N\sqrt{pq}/\sqrt{n})}{2np\big(1-N\sqrt{q}/\sqrt{np}-k/(np)\big)}\right \}}.  
\end{eqnarray} 

The estimate \eqref{00as3} is valid for $\ell\in \Delta_{n}(N)$ without additional conditions and can be written as $\Big(\frac{l}{n}\Big)^{k}\exp\Big\{{-}kD\Big(\frac{k}{l},\frac{k}{n}\Big)\Big\}\geq J_{M,n,k,p,N}$.
By conditions~\eqref{010e} and \eqref{010e00}, for $l\in\Delta_{n}(N)$ the last inequality and representation \eqref{main3} entail the bounds \eqref{1main30} and the first part of \eqref{1main3000}.

The bounds \eqref{as5}, \eqref{e01} and \eqref{00as3} and the representations \eqref{main3} and \eqref{111}  entail the inequality~\eqref{1main3}.
Taking into account the inequalities  \eqref{1main3} and 
\begin{equation} \label{00as}
p_{M}^{(\lambda_{M})}(n,k,p)\geq\mathbf{E}[p_{M}^{(\lambda_{M})}(n,\ell,k,p);\ell\in \Delta_{n}(N)],
\end{equation}
the conditional probability formula and relation \eqref{1}, leads to bound~\eqref{0main3}.

Finally, the bounds \eqref{0main30} and the second part of \eqref{1main3000} follow from the inequalities \eqref{1main30} together with the first part of \eqref{1main3000} and \eqref{00as}.
\end{proof}

In order to formulate a lower bound for the probability of ``successful'' crossover,
in Lemma~\ref{le01fc} below, we introduce
\begin{definition} \label{def:J_CnkpcN} Let
$$
J_{C,n,k,p,c,N}:=c^{k}(1-c)^{np+N\sqrt{npq}-k}.
$$
\end{definition}

Now we present a simple uniform estimate for $p_{C}^{(\lambda_{C})}(n,l,k,p,c)$. It can be strengthened, but this requires additional calculations and will not change the asymptotic behavior of its leading term.
\begin{lemma} \label{le01fc}
Suppose that $k\in\mathcal{P}_{k,p}(c_{1})$, $l\in\Delta_{n}(N)$, $p\in(0,p_{0})$ with any fixed $c_{1}\in[0,1)$, $N\in\mathbb{N}$, $p_{0}\in(0,1)$.  Let $k\to\infty$, $np\to\infty$, $l\in\Delta_{n}(N)\subseteq [k+1,n]$ and 
%the following bound holds
\begin{equation} \label{010c}
\lambda_{C}J_{C,n,k,p,c,N}=o(1),
\end{equation}
then for $n\to\infty$ the following asymptotic bound holds
\begin{equation} \label{1main03}
p_{C}^{(\lambda_{C})}(n,l,k,p,c)\geq \lambda_{C}J_{C,n,k,p,c,N}(1+o(1)).
\end{equation}

The right-hand side of the expression \eqref{1main03} will be  optimal for estimating the average number of fitness evaluations when the following condition  is satisfied:
\begin{equation}\label{m3ain3}
cp= kn^{-1}(1+N\sqrt{q}/\sqrt{np})^{-1}=kn^{-1}(1+o(1)).
\end{equation}

In the case $cp=kn^{-1}(1+o(1))$, the following inequalities hold:
\begin{eqnarray} \label{m1ain3}
&e^{-k(1+o(1))/(1-c)}\leq(1-c)^{np}\leq e^{-k(1+o(1))},\\ \label{m01ain3}
&e^{-N\sqrt{ck}(1+o(1))/(1-c)}\leq(1-c)^{N\sqrt{npq}}\leq e^{-N\sqrt{ck}(1+o(1))}.
\end{eqnarray}
\end{lemma}
\begin{proof}
Due to the monotone decrease of the function $(1-c)^{x}$, in view of the representation~\eqref{Clam}, and the condition $\lambda_{C}c^{k}(1-c)^{l-k}= o(1)$ for $l\in\Delta_{n}$, with $n\to\infty$ we have
\begin{eqnarray*}
&&p_{C}^{(\lambda_{C})}(n,l,k,p,c)=1-\exp\big\{\lambda_{C}\ln\big(1-p_{C}^{(1)}(n,l,k,p,c)\big)\big\}\\
&=&1-\exp\big\{-\lambda_{C}c^{k}(1-c)^{l-k}(1+o(1))\big\} \\
&=& \lambda_{C}c^{k}(1-c)^{l-k}(1+o(1))\geq \lambda_{C} c^{k}(1-c)^{np+N\sqrt{npq}-k}(1+o(1)).
\end{eqnarray*}

Let us find the maximum of the function $c^{k}(1-c)^{np+N\sqrt{npq}-k} $ for $c\in(0,1)$, which, according to~\eqref{TI0}, minimizes one of the factors in the upper bound for the average number of fitness evaluations till reaching the optimum, $\mathbf{E}[T_{F}]$. Setting its derivative to zero and canceling out the common factors,
we obtain $k(1-c)=c(np+N\sqrt{npq}-k)$ or
$$
c=k/(np+N\sqrt{npq})= (1+N\sqrt{q})/\sqrt{np})^{-1}k/(np) ,
$$
which proves condition~\eqref{m3ain3}.

In the case $cp=kn^{-1}(1+o(1))$, taking into account the bounds \eqref{ln1}, we obtain
\begin{equation*}
e^{-\frac{k(1+o(1))}{1-c}}\leq(1-c)^{np}=e^{\ln(1-c)np}\leq e^{-cnp}=e^{-k(1+o(1))},
\end{equation*}
which proves inequalities~\eqref{m1ain3}. Bounds~\eqref{m01ain3} are proved similarly.
\end{proof}

Let us proceed to the description of the asymptotics of $p_{\lambda_{M},\lambda_{C}}(n,l,k,p,c)$ under the conditions of Lemmas~\ref{le01} and~\ref{le01fc}.
\begin{lemma} \label{le01f}
Suppose that $k\in\mathcal{P}_{k,p}(c_{1})$, $l\in\Delta_{n}(N)$, $p\in(0,p_{0})$ with any fixed $c_{1}\in[0,1)$, $N\in\mathbb{N}$, $p_{0}\in(0,1)$.  Let $k\to\infty$, $np\to\infty$, $l\in\Delta_{n}(N)\subseteq [k+1,n]$ and bounds \eqref{010} and \eqref{010c} hold, then  the following asymptotic bounds hold as $n\to\infty$:
\begin{eqnarray} \label{main300}
p_{\lambda_{M},\lambda_{C}}(n,l,k,p,c)&\geq& \lambda_{M}\lambda_{C}J_{M,n,k,p,N}J_{C,n,k,p,c,N} (1+o(1)), \\ \label{mai1n3}
p_{\lambda_{M},\lambda_{C}}(n,k,p,c)&\geq& c_{\Phi}(N) \lambda_{M}\lambda_{C}J_{M,n,k,p,N}J_{C,n,k,p,c,N}(1{+}o(1)).
\end{eqnarray}

Assuming that $\lambda_{M}=\lambda_{C}=\lambda$, and  the conditions~\eqref{010e} or {\eqref{010e00} are satisfied instead of constraint~\eqref{010}, then   we have the inequalities}
\begin{eqnarray} \label{0m1ain3}
p_{\lambda_{M},\lambda_{C}}(n,l,k,p,c)\geq \lambda(1-e^{-c_{M}^{(0)}})c^{k}(1-c)^{np+N\sqrt{npq}-k}(1+o(1)),\quad \\ \label{m1ai1n3}
p_{\lambda_{M},\lambda_{C}}(n,k,p,c)\geq c_{\Phi}(N) \lambda(1-e^{-c_{M}^{(0)}})c^{k}(1-c)^{np+N \sqrt{npq}-k}(1+o(1)), 
\end{eqnarray}
under the condition \eqref{010e} or
\begin{eqnarray} \label{0m1ain300}
&p_{\lambda_{M},\lambda_{C}}(n,l,k,p,c)\geq \lambda c^{k}(1-c)^{np+N\sqrt{npq}-k}(1+o(1)),\\ \label{m1ai1n300}
&p_{\lambda_{M},\lambda_{C}}(n,k,p,c)\geq c_{\Phi}(N)\lambda c^{k} (1-c)^{np{+}N\sqrt{npq}-k}(1+o(1))
\end{eqnarray}
under the condition \eqref{010e00}.
\end{lemma}
\begin{proof} The inequalities \eqref{main300} -- \eqref{m1ai1n3} follow from estimations \eqref{1lam0l1} -- \eqref{TI11} and Lemmas \ref{le10} -- \ref{le01fc}.
\end{proof}

\subsection{Asymptotics of $\mathbf{E}[T_{I}]$ and $\mathbf{E}[T_{F}]$ in the General Case}

In view of relations \eqref{TI0}, the assertion of Lemma~\ref{le01f} 
implies a generalization of Theorem~8 from~\cite{ADK22} in the case 
of $np\to\infty$:  

\begin{theorem} \label{Theo1} $(I)$~Suppose that $k\in\mathcal{P}_{k,p}(c_{1})$,  $p\in(0,p_{0})$ with any fixed $c_{1}\in[0,1)$, $N\in\mathbb{N}$, $p_{0}\in(0,1)$.  Let $k\to\infty$, $np\to\infty$ and the conditions \eqref{010} and \eqref{010c} are satisfied, then the inequalities
\begin{eqnarray*} 
\mathbf{E}[T_{I}]&\leq& \dfrac{(1+o(1))}{c_{\Phi}(N) \lambda_{M}\lambda_{C}J_{M,n,k,p,N}J_{C,n,k,p,c,N}}\\ 
\mathbf{E}[T_{F}]&\leq& \dfrac{(\lambda_{M}+\lambda_{C})(1+o(1))}{c_{\Phi}(N) \lambda_{M}\lambda_{C}J_{M,n,k,p,N}J_{C,n,k,p,c,N}}.
\end{eqnarray*}
hold.

$(II)$~If we additionally assume that $\lambda_{M}=\lambda_{C}=\lambda$, and conditions \eqref{010e} and \eqref{010c} are satisfied, then  the inequalities
\begin{eqnarray} \label{PTh3}
&\mathbf{E}[T_{I}]\leq\Big(c_{\Phi}(N)\lambda c^{k}(1-c)^{np+N\sqrt{npq}-k}\big(1-e^{-c_{M}^{(0)}}\big)\Big)^{-1}(1+o(1)),\\ \label{PTh300}
&\mathbf{E}[T_{I}]\leq\Big(c_{\Phi}(N)\lambda c^{k}(1-c)^{np+N\sqrt{npq}-k}\Big)^{-1}(1+o(1)),
\end{eqnarray}
hold  under the conditions \eqref{010e} and \eqref{010e00}, accordingly. Besides that, the inequality
\begin{equation*}
\mathbf{E}[T_{F}]\leq 2\lambda\mathbf{E}[T_{I}]
\end{equation*}
holds.

\end{theorem}

The case of uniformly bounded $np$ is not considered here because it must be investigated using other asymptotic methods, in particular, with the help of Poisson approximations.

In order to discuss the difference between the Theorem~\ref{Theo1} above and the
Theorem~8 from~\cite{ADK22}, let us reproduce the formulation 
of the latter.

\begin{theorem}[Theorem 8 \cite{ADK22}] \label{Theo1Do}
Let $k\leq n/4$. Assume that $p\geq2k/n$ and $q_{\ell}=\mathbf{P}(\{\ell\geq pn\}\cap \{\ell \leq 2pn\})>0.1$. Then the expected runtime of $(1 + (\lambda, \lambda))$ GA on $\jump_{k}$ is
\begin{equation} \label{A22BK}
\mathbf{E}[T_{I}]\leq\dfrac{4}{q_{l}\min\{1,\lambda_{M}(p/2)^{k}\}\min\{1,\lambda_{C}c^{k}(1-c)^{2pn-k}\}}
\end{equation}
iterations and
\begin{equation} \label{A22BFK}
\mathbf{E}[T_{F}]\leq\dfrac{4(\lambda_{M}+\lambda_{C})}{q_{l}\min\{1,\lambda_{M}(p/2)^{k}\}\min\{1,\lambda_{C}c^{k}(1-c)^{2pn-k}\}}
\end{equation}
fitness evaluations if the algorithm starts in the local optimum.
\end{theorem}

Let us compare the statement of Theorem~\ref{Theo1} above
and Theorem~8 from~\cite{ADK22} using the example of $\mathbf{E}[T_{I}]$. 
Let us note the advantages of the obtained formula~\eqref{PTh3} in comparison with the one just given in \eqref{A22BK}.
Instead of $(p/2)^{k}$, we use $ J_{M,n, k,p,N}$, which can be estimated in the case $k=o(np)$ by the value $p^{k}e^{-Ck^{2}/(np)}\gg (p/2)^{k}$, where $C$ is a certain constant.
Furthermore, instead of $(1-c)^{2pn-k}$ from Theorem~8~\cite{ADK22}, we use use an asymptotically greater value $(1-c)^{np+N\sqrt{npq}-k}$. Instead of the estimate $\min\{1,a\}$, we use the more precise formula $1-e^{-a}$. The constant $4/q_{l}$ corresponds to $c_{\Phi}^{-1}(N)$ here, which can be chosen arbitrarily close to 1. In addition, all constants in the bounds are given explicitly here. 
%The bounds from Theorem \ref{Theo1} significantly refine their analogs from \cite{ADK22}. 
Naturally, our upper bounds for $\lambda_{M}$ and $\lambda_{C}$, leading to infinitesimal $p_{\lambda_{M},\lambda_{C}}(n,k,p,c) $ are significantly lower than the analogs from~\cite{ADK22}.

A significant difference between the studies in \cite{ADK22} and ours is the range for $l$ in the estimates of $p_{M}^{(1)}(n,l,k,p)$ and $p_{C}^{(1)}(n,l,k,p,c)$. In \cite{ADK22}, it is $np\leq l\leq 2np$, while in our study, it $l\in\Delta_{n}(N)$.
It is obvious that if we additionally require the condition $l\geq np$ in the conditions of Theorem 1, then all components containing the factor $N$ should be removed from the conditions and statements, and $c_{\Phi}(N)$ should be replaced with $0.5c_{\Phi}(N)$. In this version of Theorem, our result remains more precise than Theorem 8 from \cite{ADK22}. The effect of changing the lower bound for $l$ in the interval $[np-N\sqrt{npq},np]$ is easy to investigate, but we will not do so. We will only note that some of the components in the average estimate increase, while others decrease.

Let us consider a very simple case of estimates for the means, where the parameter $k$ is finite and does not depend on the others that depend on $n\to\infty$. These estimates are obtained using the central limit theorem.

Set $\Delta:=\sqrt{npq}$ $B=\Delta_{n}(N_{1},N_{2}):=[np - N_{1}\Delta,np + N_{2}\Delta]$, where $N_{1},N_{2}\in\mathbb{R}^{+}$. 
By the de Moivre–Laplace Theorem (see e.g. Theorem~5.3.1 from~\cite{Bo99})
for $N_{1},N_{2}=o\big((npq)^{1/6}\big)$ (possibly dependent on $n$)  we have
\begin{equation} \label{1aa}
\mathbf{P}(\ell\in\Delta_{n}(N_{1},N_{2}))\sim \Phi(N_{1})+\Phi(N_{2})-1=: c_{\Phi}(N_{1},N_{2}),
\end{equation}
where $\Phi(N)$ is the cumulative distribution function of the standard normal random value. 

\begin{lemma} \label{Le001}
Suppose that $p=p(n)\in(0,p_{0})\subseteq (0,1)$, $c\in (0,c_{0})\subseteq (0,1)$, where $p_{0}>0$ and $c_{0}>0$ are fixed, and $np\to\infty$. 
Then for any fixed $k$, $2\leq k\in\mathbb{N},$ 
%independent of $n$ 
the following bounds hold  
\begin{eqnarray} \label{0PTh1}
&\mathbf{E}[T_{I}]\leq\dfrac{c_{\Phi}^{-1}(N_{1},N_{2})(1+o(1))}{\big(1-e^{-\lambda_{M}p^{k}(1+o(1))}\big)\big(1-e^{-\lambda_{C}c^{k}(1-c)^{np+N_{2}\Delta-k}(1+o(1))}\big)},
\\ \label{0PTh2}
&\mathbf{E}[T_{F}]\leq (\lambda_{M}+\lambda_{C})\mathbf{E}[T_{I}],
\end{eqnarray}
where $N_{1},N_{2}=o\big((npq)^{1/6}\big)$.
\end{lemma}
\begin{proof} Given the definitions \eqref{main4} and \eqref{Clam}, by the monotone increasing of $p_{M}^{(1)}(n,l,k,p)$ decreasing $p_{C}^{(1)}(n,l,k,p,c)$ with respect to $l$, we obtain the estimates $p_{M}^{(1)}(n,l,k,p)\geq p_{M}^{(1)}(n,np-N_{1}\Delta,k,p)$ and $p_{C}^{(1)}(n,l,k,p,c)\  \geq p_{C}^{(1)}(n,np +N_{2}\Delta,k,p,c)$
for all $l\in\Delta_{n}(N_{1},N_{2})$. Therefore, taking into account the inequalities \eqref{e1}, we obtain:
\begin{eqnarray} \nonumber
&&p_{\lambda_{M},\lambda_{C}}(n,l,k,p,c)\geq \left (1-e^{-\lambda_{M}p_{M}^{(1)}(n,np-N_{1}\Delta,k,p)(1+o(1))}\right )\\\label{1PTh2}
&\cdot&\left (1-e^{-\lambda_{C}p_{C}^{(1)}(n,np +N_{2}\Delta,k,p,c)(1+o(1))}\right).
\end{eqnarray}

Note that by the definitions \eqref{main4} for $n\to\infty$ and fixed $k$ and $N_{1}>0$,  denoting $l_{0}:=np-N_{1}\Delta$, we obtain
the following representation:
\begin{equation}\label{333}
p_{M}^{(1)}(n,l_{0},k,p)=n^{-k}l_{0}^{k}(1+o(1))=p^{k}(1+o(1)).
\end{equation}
Statements \eqref{0PTh1} and \eqref{0PTh2} are consequences of relations \eqref{TI0}, \eqref{TI11}, \eqref{333} and \eqref{1aa} when the values of probabilities corresponding to the lemma statement are substituted.
\end{proof}
 
We will not go into details optimizing the results of 
Theorem~\ref{Theo1} with respect to the parameters $\lambda_{M}$ 
and $\lambda_{C}$. Instead, in 
Subsections~\ref{subsec:finite_k} and~\ref{subsec:infinite_k} below 
we will use the conditions of Corollary~9 \cite{ADK22} with the 
additional assumption $np\to\infty$ and present in Theorems~\ref{Theo01} and~\ref{Theo2} two analogs of the 
Corollary~9 from~\cite{ADK22},
%The following two theorems are analogous to Corollary~9  
%from~\cite{ADK22}, 
where instead of abstract upper bounds on the 
parameters $\lambda_{M}$ and $\lambda_{C}$, we choose 
their explicit asymptotic representations. As a result, instead 
of upper bounds up to an unknown constant, we obtain explicit 
expressions. 
It is natural to divide the general case into 
two: $k$ is finite (which is the same as $k$ is bounded uniformly w.r.t. $n$) and $k\to\infty$.

\subsection{Asymptotics of $\mathbf{E}[T_{I}]$ and $\mathbf{E}[T_{F}]$ for finite $k$ and $\lambda_{M}=\lambda_{C}=\lambda$}
\label{subsec:finite_k}

\begin{theorem} \label{Theo01}
Assume that $2\leq k\in\mathbb{N}$ is fixed, $\lambda_{M}=\lambda_{C}=\lambda=\alpha \sqrt{n/k}^{k}$, where $\alpha =\alpha (n)$, and $p=c=\sqrt{k/n}$, then 

$(I)$ in the case  of $\varliminf\limits_{n\to\infty}\alpha (n)>0$ we have
\begin{eqnarray} \label{0PTh3}
&\mathbf{E}[T_{I}]\leq\dfrac{1+o(1)}{\big(1-e^{-\alpha (1+o(1))}\big)\big(1-\exp\{-e^{-k} \alpha (1+o(1))\}\big)},\\ \label{0PTh4}
&\mathbf{E}[T_{F}]\leq 2\lambda\mathbf{E}[T_{I}]=2\alpha \sqrt{n/k}^{k}\mathbf{E}[T_{I}],
\end{eqnarray} 

$(II)$ in the case of
%$\lambda\in\mathbb{N}$ and 
$\alpha  (n)=o(1)$ we have
\begin{equation} \label{0PTh03}
\mathbf{E}[T_{I}]\leq \dfrac{ e^{k}(1+o(1))}{\alpha^2},
\quad \mathbf{E}[T_{F}]\leq \dfrac{2e^{k}\sqrt{n/k}^{k}(1+o(1))}{\alpha}.
\end{equation}
\end{theorem}
\begin{proof}
Set $N_{1}=N_{2}=N=o\big((npq)^{1/6}\big)$, $N\to\infty$. 
Then by the estmation \eqref{1} we have
 $\mathbf{P}(l\in\Delta_{n}(N,N))=\mathbf{P}(l\in\Delta_{n}(N))=1+o(1)$.
Taking into account the assumptions of this theorem, 
it is easy to obtain the following relations
\begin{eqnarray*}
\lambda_{M}p^{k}(1+o(1))&=&\alpha (1+o(1)),
\\ 
\lambda_{C}c^{k}(1-c)^{np+N_{2}\Delta-k}&=&\alpha e^{np\ln(1-c)(1+o(1))}\\
&=&\alpha e^{-cnp(1+o(1))}=\alpha e^{-k}(1+o(1)),
\end{eqnarray*}
where $N_{2}\Delta=o(np)$. So,
$$
1-e^{-\lambda_{C}c^{k}(1-c)^{np+N_{2}\Delta-k}}=1-\exp\{-e^{-k}\alpha (1+o(1))\}.
$$

By the latter relations,  inequalities \eqref{0PTh3} and
 \eqref{0PTh4}  follow from bounds \eqref{0PTh1} and \eqref{0PTh2} of Lemma~\ref{Le001}.
 Finally, in case $\alpha  (n)=o(1)$ we have 
\eqref{0PTh03}.
\end{proof}

Computational experiment performed in~\cite{GEC26} 
for the values of parameters $k=2,\dots,6$, $n=2^6,2^7,\dots,2^{10}$ indicates agreement of the 
upper bound from Theorem~\ref{Theo01}, case~$(I),$ 
$\alpha=1,$ with experimental results,  even if $n$ is finite.

\subsection{Asymptotics of $\mathbf{E}[T_{I}]$ and $\mathbf{E}[T_{F}]$ for $k\to\infty$ and $\lambda_{M}=\lambda_{C}=\lambda$}
\label{subsec:infinite_k}

Let us introduce the notation for $\lambda_{M} J_{M,n,k,p,N}$ and $\lambda_{C}J_{C,n,k,p,c,N}$
in the case of $\lambda_{M}=\lambda_{C}=\lambda=\alpha \sqrt{n/k}^{k}$ and $p=c=\sqrt{k/n}$, 
\begin{eqnarray} \label{00m1ain3}
I_{M,n,k,N}&:=&\alpha \sqrt{n/k}^{k}J_{M,n,k,\sqrt{k/n},N},\\\label{01m1ai1n3}
I_{C,n,k,N}&:=&\alpha \sqrt{n/k}^{k}J_{C,n,k,\sqrt{k/n},\sqrt{k/n},N}. 
\end{eqnarray}

In the following lemma, we provide explicit expressions for the sequences defined in relations \eqref{00m1ain3} and \eqref{01m1ai1n3}, followed by their rough estimates of the type $e^{f(y)(1+o(1))}$ and precise estimates of the type $e^{f(y)}(1+o(1))$ under a number of additional conditions.

\begin{lemma} \label{le01fmc}
Suppose that $k\in\mathcal{P}_{k,p}(c_{1})$,  $p\in(0,p_{0})$ with any fixed $c_{1}\in[0,1)$, $N\in\mathbb{N}$, $p_{0}\in(0,1)$.  Let $\lambda_{M}=\lambda_{C}=\lambda=\alpha \sqrt{n/k}^{k}$ and $c=c(k,n)=p=p(k,n)=\sqrt{k/n}$, $k\to\infty$, then in the case  of $\varlimsup\limits_{n\to\infty}p<1$, the following 
represetations hold 
\begin{eqnarray}  \label{M1es1}
I_{M,n,k,N}=\alpha\sqrt{1+p}\exp\left \{ -N\sqrt{kpq}-\dfrac{N^{2}pq}{2}-\dfrac{kp(q+Np \sqrt{pq/k})}{2(q-N\sqrt{qp/k})}\right \},\\ \label{M1es2}
I_{C,n,k,N}=\alpha \exp\left\{\ln(1-p)\big((1-p)k/p+N\sqrt{(1-p)k/p}\big)\right \}.\qquad\qquad
\end{eqnarray}

If $k^{2}/n=\tilde{c}+o(1)$ for any fixed constant $\tilde{c}\in[0,\infty)$, 
then we have the following estimate  
\begin{eqnarray} \label{M1est200}
I_{M,n,k,N}I_{C,n,k,N}=\alpha^{2} \exp\left \{-k+\tilde{c}/6-2N\sqrt{kp} \right\}(1+o(1)).
\end{eqnarray}

In the case of $p^{2}=k/n=o(1)$ the following rough estimate is true  
\begin{equation}  \label{Mest2a} 
I_{M,n,k,N}I_{C,n,k,N}\geq \alpha ^{2}e^{-k(1+o(1))}.
\end{equation}
\end{lemma}
\begin{proof}
By the definitions of $I_{B,n,k,N}$ and $I_{C,n,k,N}$, substituting $\lambda_{M}=\lambda_{C}=\lambda=\sqrt{n/k}^{k}$ and   $p=c=\sqrt{k/n}$ and taking into account the identities $np=\sqrt{kn}$ and  
\begin{equation} \label{est2}
\sqrt{\dfrac{1-k/n}{1-p}}=\sqrt{\dfrac{1-k/n}{1-\sqrt{k/n}}}=\sqrt{1+\sqrt{k/n}},
\end{equation}
we obtain the equalities \eqref{M1es1} and  \eqref{M1es2}.

If $\sqrt{k/n}=p=o(1)$, then by the Taylor formula $\ln(1-p)=-p-0.5p^{2}-p^{3}(1+o(1))/3$ and $\sqrt{q} = \sqrt{1-p}=1-0.5p(1+o(1))$.

Note that for $p\to0$
\begin{equation*}
 \sqrt{kp}p=\left \{ \begin{array}{ll}
 o(1),&kp=O(1),\\
o(kp^{2}),& kp\to\infty
 \end{array},
 \right .
\end{equation*}

Using Taylor's formula for $k^{2}/n=\tilde{c}+o(1)$, $\tilde{c}\in[0,\infty)$ independent of $n$, we obtain estimates
\begin{eqnarray*}
W_{1}&:=& -N\sqrt{kpq}-\dfrac{N^{2}pq}{2}-\dfrac{kp(q+Np \sqrt{q/k})}{2(q-N\sqrt{qp/k})}\\
&=&-N\sqrt{kp}-\dfrac{kp(1-p+Np\sqrt{p/k})}{2(1-p-N\sqrt{p/k})}+o(kp^{2})+o(1)\\
&=&-N\sqrt{kp}-0.5kp+o(kp^{2})+o(1)
\end{eqnarray*}
and
\begin{eqnarray*}
&&\hspace{-0.6cm}W_{2}:=\ln(1-p)\big(pqn+
N\sqrt{pqn}\big)
=-\big(p+0.5p^{2}+p^{3}(1+o(1))/3\big)\\
&\cdot&(1-p)k/p-
\big(p+0.5p^{2}(1+o(1))\big)N\sqrt{(1-p)k/p}\big)\\ 
&=&-k\big(1-0.5p-p^{2}/6\big)-N\sqrt{kp}+o(kp^{2})+o(1).
\end{eqnarray*}

Therefore, the relations \eqref{M1es1} and \eqref{M1es2} for $k^{2}/n=\tilde{c}+o(1)$ take the form 
\begin{eqnarray} \label{M1est2}
I_{M,n,k,N}&=&\alpha e^{-N\sqrt{kp}-0.5kp}(1+o(1)),\\\label{M1est2b}
I_{C,n,k,N}&=&\alpha e^{-k\big(1-0.5p-p^{2}/6\big)-N\sqrt{kp}}(1+o(1))
\end{eqnarray}
and imply \eqref{M1est200}.

We will not present the explicit formula for the product $hj$ that follows from relations \eqref{M1est2} and \eqref{M1est2b} instead, we will focus on the special case $k^{2}/n=O(1)$.

In the case $p=o(1)$, the relations \eqref{M1es1} and \eqref{M1es2} take on the following form
\begin{eqnarray*} 
I_{M,n,k,N}&=&\alpha e^{-0.5kp(1+o(1))},\\
I_{C,n,k,N}&=&\alpha e^{-k +kO(p+\sqrt{p/k})}.
\end{eqnarray*}

So by definitions with the help of the last two estimates we obtian~\eqref{Mest2a}.
\end{proof}

Let us present a special case of Theorem \ref{Theo1} where $p=c=\sqrt{k/n}$ and $\lambda_{M}=\lambda_{C}=\lambda=\alpha \sqrt{n/k}^{k}$, 
based on the results of Lemma~\ref{le01fmc}. We are focusing on the case $I_{M,n,k,N}=o(1)$ and $I_{C,n,k,N}=o(1)$, and if 
at least one of these  conditions is violated, then the average estimates are obtained using Lemma~\ref{le10}.

\begin{theorem} \label{Theo2}
Suppose that $k\in\mathcal{P}_{k,p}(c_{1})$,  $p\in(0,p_{0})$ with any fixed $c_{1}\in[0,1)$, $N\in\mathbb{N}$, $p_{0}\in(0,1)$.  Let $c=p=\sqrt{k/n}$, $\lambda_{M}=\lambda_{C}=\lambda=\alpha \sqrt{n/k}^{k}$, $k\to\infty$, then, in the case $I_{M,n,k,N}=o(1)$ and $I_{C,n,k,N}=o(1)$, the estimates hold
\begin{equation*}%\label{DT2} 
\mathbf{E}[T_{I}] 
\leq c_{\Phi}^{-1}(N)I_{M,n,k,N}^{-1}I_{C,n,k,N}^{-1}(1+o(1)),\ \ \mathbf{E}[T_{F}]\leq 2\alpha \sqrt{n/k}^{k}\mathbf{E}[T_{I}],
\end{equation*}
where for $I_{M,n,k,N}$ and $I_{C,n,k,N}$, we need to apply bounds \eqref{M1es1} -- \eqref{Mest2a} from Lemma \ref{le01fmc}.

In the case  $k/n=o(1)$ we have
\begin{equation}\label{DT02} 
\mathbf{E}[T_{I}] 
\leq \alpha ^{-2}e^{k(1+o(1))},\ \ \mathbf{E}[T_{F}]\leq 2\alpha \sqrt{n/k}^{k}\mathbf{E}[T_{I}].
\end{equation}

If we replace condition $I_{M,n,k,N}=o(1)$ with condition \eqref{010e}, 
we get the bounds 
\begin{equation*}%\label{DT2} 
\mathbf{E}[T_{I}] 
\leq \frac{(1+o(1))}{c_{\Phi}(N)I_{C,n,k,N}\left(1-e^{-c_{M}^{(0)}}\right )},\ \ \mathbf{E}[T_{F}]\leq 2\alpha \sqrt{n/k}^{k}\mathbf{E}[T_{I}],
\end{equation*}
where $I_{C,n,k,N}$ is estimated in Lemma \ref{le01fmc}.
\end{theorem}
\begin{proof} 
All statements of the theorem are obtained taking into account the relations \eqref{111} when substituting expressions from Lemma \ref{le01fmc} into statements from Theorem \ref{Theo1}.
 In estimates \eqref{DT02}, the constant $c_{\Phi}^{-1}(N)$ is absent because it is absorbed by the expression $e^{ko(1)}$.
\end{proof} 
Theorem \ref{Theo2} generalizes Corollary~10 from~\cite{ADK22}:\\

\noindent\textbf{Corollary 10 \cite{ADK22}.} \textit{Let $k\in[2,\cdots,\lfloor n/4\rfloor]$. Assume that $p=c=\sqrt{k/n}$ and $\lambda_{M}= \lambda_{C} =\lambda \leq 2^{k}\sqrt{n/k}^{k}$. Then the expected runtime of the $(1+(\lambda,\lambda))$~GA on $\jump_k$, is $\mathbf{E}[T_{F}]\leq n^{k}k^{-k}e^{O(k)}\lambda^{-1}$ fitness evaluations and $\mathbf{E}[T_{I}] \leq n^{k}k^{-k}e^{O(k)}\lambda^{-2}$ iterations, if it starts in a local optimum of $\jump_k$. For $\lambda=\sqrt{n/k}^{k}$ these bounds are $\mathbf{E}[T_{I}] \leq e^{O(k)}$ and $\mathbf{E}[T_{F}]\leq \sqrt{n/k}^{k}e^{O(k)}$.
}

\medskip
The bound $\mathbf{E}[T_{I}] \le \alpha^{-2}e^{k(1+o(1))}$ from Theorem \ref{Theo2} in the case considered in~\eqref{DT02} 
improves Corollary~9 \cite{ADK22}, where the upper bound $\mathbf{E}[T_{I}] =e^{O(k)}$ was obtained. 
We describe here the case of $\lambda \sim \alpha \sqrt{n/k}^{k}$ for $\alpha $ that are not bounded above, 
whereas in \cite{ADK22} they are bounded by $2^{n}$.

We note that our theorems give an explicit form of upper bounds for the averages, which essentially depend on the main parameters depending on $n$. In the work \cite{ADK22}, similar estimates are given without dependence on the main parameters and with accuracy to unknown bounded constants. 
%Our estimates are more precise than the ones in \cite{ADK22}, and their quality can be naturally 
%checked in computational experiments. The general approach without an explicit form of the asymptotics of the constants does not allow to estimate the quality of heterogeneous estimates. Of course, the transition to a logarithmic scale removes almost all problems, but this will be a very rough estimate of the quality of the estimates.

In what follows, we derive a corollary of Theorem \ref{Theo1}, which is analogous to Corollary 10 in \cite{ADK22}. Our results are obtained under less restrictive conditions, and the constants in the estimates are explicitly given. We exclude the case of fixed (or uniformly bounded) $k$. In this case, the statement of Corollary~10~\cite{ADK22} is of the form $\mathbf{E}[T_{F}]=o(n^{k})$. More general result in this case can be easily obtained using a natural modification of Theorem \ref{Theo01}. It is not difficult with the help of Lemma \ref{le10} to investigate the cases $\varlimsup\limits_{n\to\infty}\lambda_{M}J_{M,n,k,p,N}>0$ or $\varlimsup\limits_{n\to\infty}\lambda_{C}J_{C,n,k,p,c,N}>0$, but we omit them, to simplify the proof. To simplify the formulation of the corollary and its proof, we also exclude the case $\varlimsup\limits_{n\to\infty}c(n)>0$.

\begin{corollary} \label{ct1} 
Let conditions of part~$(I)$ of Theorem \ref{Theo1} hold,  $k=k(n)$,
% \in\mathbb{N}$, 
$k\to\infty$ and $k=o(n)$. 
Assume that $p=p(n)=\nu(n)k/n=\nu k/n=o(1)$, where ${\nu(n)\to\infty}$, $c=c(n)=\rho(n)k/n=\rho k/n=o(1)$, where $\rho(n)\to\infty$, and $pc =\sigma(n)k/n=\sigma k/n$, where $\sigma(n)=O(1)$. In addition, $\lambda_{M}p^{k}=\tau_{1}(k/(nc))^{k}$, $\lambda_{C}c^{k}=\tau_{2}(k/(np))^{k}$, where $\tau_{i}=\tau_{i}(n)\to\infty$, $i=1,2$.

Then the $(1 + (\lambda,\lambda))$ GA starting from a local optimum of $\jump_k$ has the expected runtime in terms of iterations
\begin{equation} \label{er1} 
\mathbf{E}[T_{I}]\leq \dfrac{(n\sigma/k)^k}{\tau_{1}\tau_{2}}\exp\left \{\dfrac{k(1+o(1))}{2\nu}+k\sigma(1+o(1))\right\}(1+o(1)),
\end{equation}
and the expected runtime in terms of fitness evaluations 
\begin{equation} \label{er3}
\mathbf{E}[T_{F}]\leq \dfrac{\tau_{1}+\tau_{2}}{\sigma^{k}}\mathbf{E}[T_{I}].
\end{equation}
\end{corollary}
\begin{proof} 
Conditions \eqref{010} and \eqref{010c} are analogous to the inequalities $\alpha \leq 1$ and $\beta\leq 1$ from Corollary 10  \cite{ADK22}.
From the representations \eqref{cole01} and \eqref{010c}, the estimates 
\begin{eqnarray} \label{c1} 
J_{M,n,k,p,N}&=&
p^{k}\exp\left \{\dfrac{-k(1+o(1))}{2\nu}\right \},
\\  \label{c2}
J_{C,n,k,p,c,N}&=&c^{k}e^{-k\sigma(1+o(1))}.
\end{eqnarray}
are valid under the conditions of the corollary.
When deriving the latest estimates, we used the following relationships: $k/(np)=1/\nu=o(1)$, $q=1+o(1)$ and $c(np+N\sqrt{npq}-k)=cnp(1+o(1))=k\sigma(1+o(1))$.

By definitions $\lambda_{M}=\tau_{1}/\sigma^{k}$ and $\lambda_{C}=\tau_{2}/\sigma^{k}$.

Therefore, under the representations \eqref{c1} and \eqref{c2}, in view of Theorem~\ref{Theo1}, the estimate \eqref{er1} holds.
To complete the proof, it suffices to apply the representations \eqref{TI0}.
\end{proof} 

Corollary \ref{ct1} generalizes the following result.\\

\noindent\textbf{Corollary 10 \cite{ADK22}.} \textit{Let $k\geq2$ and $k=o(n)$. Assume that $p=\omega(k/n)$, $c=\omega(k/n)$ and $pc=O(k/n)$. 
Define $\alpha:=\lambda_{M}(p/2)^{k}$ and $\beta:=\lambda_{C}c^{k}(1-c)^{2pn-k}$.
If $\alpha$ and $\beta$ are at most one and $\alpha=\omega((k/(nc))^{k})$ and $\beta=\omega((2k/(np))^{k})$, then the expected number
of fitness evaluations until the $(1+(\lambda,\lambda))$~GA reaches the global optimum
starting from a local optimum of $\jump_k$ is}
$$
\mathbf{E}[T_{F}]=o\big( n^{k}k^{-k}\big)e^{O(k)}.
$$ 

\medskip
The main difference in the statement of Corollary \ref{ct1} and Corollary 10 \cite{ADK22} is the replacement of the expressions comparing two sequences of the type $u(n)=o(w(n))$ and $u(n)= O(w(n))$ by the functions $z(n)=u(n)/w(n)$ with the corresponding restrictions on $z(n)$. This allows us to obtain explicit estimates in terms of $z(n)$ instead of estimates in terms of $O(\cdot)$.

In~\cite{ADK22}, after the formulation of the statement of Corollary~10, it was explained that the conditions of the latter imply the constraints $p=o(1)$ and $c=o(1)$, which we include in the conditions of Corollary \ref{ct1}.

Here is a consequence of Theorem~\ref{Theo1}, which is analogous to Theorem~11 from~\cite{ADK22}. We also exclude the simple case of $\varlimsup\limits_{n\to\infty}k(n)<\infty$, which may be studied by other standard methods.

\begin{theorem} \label{Theo5}
Let $k\in\mathbb{N}$, $k\to\infty$ and $k\leq n/4$. Assume that $p=\lambda/n$, $c=1/\lambda$
and $\lambda_{M}=\lambda_{C}=\lambda$
for some $\lambda\in\mathbb{N}$, $\delta k\leq \lambda  \leq n$ for some fixed $\delta>1$.

Then
\begin{eqnarray}\label{ee} % \nonumber
\mathbf{E}[T_{I}]\leq \frac{n^{k}}{\lambda^{2}}  \left\{
\begin{array}{ll}
\dfrac{e(1+o(1))}{c_{\Phi}(N)} , &\text{if\ }  \frac{k}{\sqrt{\lambda}}=o(1),\\
\sqrt{\frac{1-k/\lambda}{1-k/n}} \exp\left \{\frac{k^{2}q(1+o(1))}{2\lambda(1-k/\lambda)}\right \}, &\text{if\ } \frac{k}{\sqrt{\lambda}}\to\infty,\\
\dfrac{e^{c_{0}^{*}N\sqrt{q}+0.5{c_{0}^{*}}^{2} q+1}}{c_{\Phi}(N)}(1+o(1)), &\text{if\ } \frac{k}{\sqrt{\lambda}}{=}c_{0}^{*}+o(1),
\end{array}
\right .
\end{eqnarray}
\end{theorem}
\begin{proof} 
If the conditions of the theorem are met, the identity 
\begin{equation} \label{cole0100}
J_{M,n,k,p,N}=\frac{\sqrt{\frac{1-k/n}{1-k/\lambda}} p^{k}}{\exp\left \{\dfrac{kN\sqrt{q}}{\sqrt{\lambda}}+\dfrac{kN^{2}q}{2\lambda}+\dfrac{k^{2}(q+N\sqrt{q\lambda}/n)}{2\lambda\big(1-N\sqrt{q}/\sqrt{\lambda}-k/\lambda \big)}\right \}}
\end{equation}
holds true.

When proving an estimate of type \eqref{ee}, identity \eqref{cole0100} can be used, but the result will be very cumbersome. Therefore, in estimate \eqref{ee} we single out only a number of special cases with compact relations in the right-hand side.

From the representations \eqref{cole0100} and \eqref{010c}, the estimates 

\begin{eqnarray}  \label{ee11}
J_{M,n,k,p,N}=\left(\frac{\lambda}{n}\right )^{k}\left\{
\begin{array}{ll}
(1+o(1)) , &\text{if\ }  \frac{k}{\sqrt{\lambda}}=o(1),\\
\sqrt{\frac{1-k/n}{1-k/\lambda}} \exp\left \{{-}\frac{k^{2}q(1+o(1))}{2\lambda(1-k/\lambda)}\right \}, &\text{if\ } \frac{k}{\sqrt{\lambda}}\to\infty,\\
e^{-c_{0}^{*}N\sqrt{q}-0.5{c_{0}^{*}}^{2} q}(1+o(1)), &\text{if\ } \frac{k}{\sqrt{\lambda}}{=}c_{0}^{*}+o(1),
\end{array}
\right .\\  \label{ee12}
J_{C,n,k,p,c,N}=\lambda^{-k}e^{k/\lambda-1}(1+o(1)).\qquad\qquad\qquad\qquad
\end{eqnarray}
are valid under the conditions of the theorem.
In the last representation we take into account the sequence of 
inequalities
\begin{eqnarray*}
&&(1-c)^{np+N\sqrt{npq}-k}=\exp\{\ln(1-c)(np+N\sqrt{npq}-k)\}\\
&=&\exp\{-c(np+N\sqrt{npq}-k)(1+o(1))\}=
e^{-1+k/\lambda}(1+o(1)), \ \lambda\to\infty.
\end{eqnarray*}

Therefore, using representations \eqref{ee11} and \eqref{ee12}  
we obtain
\begin{equation*} 
p_{\lambda_{M},\lambda_{C}}(n,l,k,p,c)\geq 
\frac{\lambda^{2}}{n^{k}}\left\{
\begin{array}{ll}
e(1+o(1)) , &\text{if\ }  \frac{k}{\sqrt{\lambda}}=o(1),\\
\sqrt{\frac{1-k/n}{1-k/\lambda}} \exp\left \{{-}\frac{k^{2}q(1+o(1))}{2\lambda(1-k/\lambda)}\right \}, &\text{if\ } \frac{k}{\sqrt{\lambda}}\to\infty,\\
e^{-c_{0}^{*}N\sqrt{q}-0.5{c_{0}^{*}}^{2} q+1}(1{+}o(1)), &\text{if\ } \frac{k}{\sqrt{\lambda}}{=}c_{0}^{*}{+}o(1),
\end{array}
\right .
\end{equation*}

To complete the theorem proof we need to use the representations \eqref{TI0}. The first and third lines of relations \eqref{ee} do not have an explicit dependence on $N$, i.e. they are true for any fixed $N$. Therefore, instead of  $c_{\Phi}^{-1}(N)$, we can write $1+o(1)$.
\end{proof} 
%{\color{red}
Let us formulate Theorem~11~\cite{ADK22} in order to compare it to Theorem~\ref{Theo5}.\\

\noindent\textbf{Theorem 11 \cite{ADK22}.} \textit{Let $k\in[2,\cdots,\lfloor n/4\rfloor]$. Assume that $p=\lambda/n$, $c=1/\lambda$ and $\lambda_{M}= \lambda_{C}=\lambda$ for some $\lambda\in[2k,n]$.}

\textit{If $\lambda\leq (2n)^{k/(k+1)}$, then the expected runtime of $(1+(\lambda,\lambda))$ GA on $\jump_k$ is $\mathbf{E}[T_{F}]\leq O(2^{k}n^{k}\lambda^{-2})$ fitness evaluations and $\mathbf{E}[T_{I}] \leq O(2^{k}n^{k}\lambda^{-1})$ iterations.
}

\textit{If $\lambda\geq (2n)^{k/(k+1)}$, then the expected runtime of $(1+(\lambda,\lambda))$ GA on $\jump_k$ is $\mathbf{E}[T_{F}]\leq O(\lambda^{k-1})$ fitness evaluations and $\mathbf{E}[T_{I}] \leq O(\lambda^{k})$ iterations.\\
}

The second part of the statement of Theorem 11 is a special case of its first part. Therefore, to compare its results with our Theorem~\ref{Theo5} , we need to evaluate the ratio of $2^{k}$ and our $\tilde{J}:=\mathbf{E}[T_{I}]\frac{\lambda^{2}}{n^{k}}$ from \eqref{ee}. It is obvious that $2^{k} \gg\tilde{J}$ when $k/\lambda=o(1)$. 

The case $\lambda=O(k)$ appears to  be of little interest for applications
%, and comparing the results for this case requires additional analysis,
so it is omitted here.

\section{Conclusions}

We have presented the runtime analysis of the $(1+(\lambda,\lambda))~GA$ on $\jump_k$ fitness functions from the prospective of the de Moivre–Laplace theorem. 
We assumed that the $(1+(\lambda,\lambda))~GA$ has tunable parameters of mutation rate~$p$, crossover bias~$c$, and two intermediate population sizes~$\lambda_M$ and~$\lambda_C$, and it starts from a local optimum of $\jump_k$ fitness function. 
The main result of this work (Theorem~\ref{Theo1}) is a refined upper bound on the escape time from the local optimum, known from~\cite{ADK22}. 
Different cases of asymptotic behaviour of parameters $k, p, c$ are considered in Theorems~\ref{Theo2} and~\ref{Theo5}.

Summing up the obtained results, we conclude that 
universal estimates for all admissible
parameter values of the $(1+(\lambda,\lambda))~GA$ on $\jump_k$
have a natural drawback -- they are rather crude. 
We have presented more precise
estimates, focused on specific regions of admissible parameter 
values.
In particular, we identified a rather simple case of finite~$k$
and studied it in Theorem~\ref{Theo01}.
The rest of our results are presented under the condition 
that $k \to \infty $, which also implies that $np \to \infty$.

We hope that in the future research the obtained results will be helpful 
for finding tight asymptotic expressions for the expected runtime
of  the $(1+(\lambda,\lambda))~GA$ on $\jump_k$ and
they may be used for further optimization of parameters in the 
$(1+(\lambda,\lambda))~GA$.

\end{document}

%% file: preamble.tex
\usepackage{amsfonts}
\usepackage{amsmath}
\usepackage{url}

\usepackage{xspace}                           % "intelligent" spaces
\usepackage{booktabs}                         % pretty tables
\usepackage{marginnote}

\usepackage{color} % this is disallowed by AAAI, so remember to turn off in camera-ready

\allowdisplaybreaks

\DeclareMathOperator*{\Bin}{Bin}

\usepackage{mathrsfs}

\usepackage{bbm}

\usepackage{algorithm,algorithmic}

\floatname{algorithm}{Алгоритм}

\newcommand{\leadingones}{\text{\sc LeadingOnes}\xspace}

\newcommand{\onemax}{\text{\sc OneMax}\xspace}

\newcommand{\jump}{\text{\sc Jump}\xspace}

\usepackage{enumitem}